\documentclass{article}

\usepackage{microtype}
\usepackage{graphicx}
\usepackage{subfigure}
\usepackage{booktabs} 
\usepackage{wrapfig}
\usepackage{scalerel}
\usepackage{enumitem}
\usepackage{amssymb}

\usepackage{hyperref}

\usepackage[accepted]{icml2021}
 
\usepackage{amsmath,amsfonts,bm}
\usepackage{multirow}
\usepackage{tabularx}
\usepackage{graphicx}
\usepackage{adjustbox}
\usepackage{amsthm}
\usepackage{color}

\newcommand{\mset}[1]{\left\{\kern-.5em\left\{ #1 \right\}\kern-.5em\right\}}
\newcommand{\mmset}[1]{\{\kern-.4em\{ #1 \}\kern-.4em\}}

\newcommand{\BV}{\mathrm{BV}} 

\newcommand\eqdef{\mathrel{\overset{\makebox[0pt]{\mbox{\normalfont\tiny\sffamily def}}}{=}}}
\newcommand{\divv}{\mathrm{div}} 

\newcommand{\norm}[1]{\left\Vert#1\right\Vert}

\newcommand{\abs}[1]{\left\vert#1\right\vert}

\newcommand{\set}[1]{\left\{#1\right\}}
\newcommand{\parr}[1]{\left (#1\right )}
\newcommand{\brac}[1]{\left [#1\right ]}

\newcommand{\ip}[1]{\left \langle #1 \right \rangle }
\newcommand{\Real}{\mathbb R}

\newcommand{\eps}{\varepsilon}
\newcommand{\too}{\rightarrow}

\newcommand{\dtoo}{\,{\scriptstyle{\downarrow}}\,}

\newcommand{\per}{\mathrm{per}}
\newcommand{\one}{\mathbf{1}}

\newcommand{\eg}{{e.g.}}
\newcommand{\ie}{{i.e.}}

 \makeatletter
\newtheorem*{rep@theorem}{\rep@title}
\newcommand{\newreptheorem}[2]{%
\newenvironment{rep#1}[1]{%
 \def\rep@title{#2 \ref{##1}}%
 \begin{rep@theorem}}%
 {\end{rep@theorem}}}
\makeatother

\newtheorem{theorem}{Theorem}
\newreptheorem{theorem}{Theorem}
\newtheorem{lemma}{Lemma}
\newreptheorem{lemma}{Lemma}

\def\eqref#1{equation~\ref{#1}}
\def\Eqref#1{Equation~\ref{#1}}

\def\1{\bm{1}}

\def\eps{{\epsilon}}

\def\vb{{\bm{b}}}

\def\ve{{\bm{e}}}

\def\vn{{\bm{n}}}

\def\vv{{\bm{v}}}

\def\vx{{\bm{x}}}

\def\vy{{\bm{y}}}
\def\vz{{\bm{z}}}
\def\vec1{{\bm{1}}}

\def\mW{{\bm{W}}}

\DeclareMathAlphabet{\mathsfit}{\encodingdefault}{\sfdefault}{m}{sl}
\SetMathAlphabet{\mathsfit}{bold}{\encodingdefault}{\sfdefault}{bx}{n}

\def\gA{{\mathcal{A}}}

\def\gE{{\mathcal{E}}}
\def\gF{{\mathcal{F}}}

\def\gH{{\mathcal{H}}}
\def\gI{{\mathcal{I}}}

\def\gL{{\mathcal{L}}}

\def\gN{{\mathcal{N}}}
\def\gO{{\mathcal{O}}}

\def\gS{{\mathcal{S}}}

\def\gU{{\mathcal{U}}}

\def\gX{{\mathcal{X}}}
\def\gY{{\mathcal{Y}}}

\newcommand{\E}{\mathbb{E}}

\DeclareMathOperator{\sign}{sign}

\usepackage{todonotes}

\usepackage{xr}
\icmltitlerunning{Phase Transitions, Distance Functions, and Implicit Neural Representations}

\begin{document}

\twocolumn[
\icmltitle{Phase Transitions, Distance Functions, and Implicit Neural Representations}




\begin{icmlauthorlist}
\icmlauthor{Yaron Lipman}{fb,wis}
\end{icmlauthorlist}

\icmlaffiliation{fb}{Facebook AI Research}
\icmlaffiliation{wis}{Weizmann Institute of Science}

\icmlcorrespondingauthor{Yaron Lipman}{ylipman@fb.com, yaron.lipman@weizmann.ac.il}

\icmlkeywords{Machine Learning, ICML}

\vskip 0.3in
]



\printAffiliationsAndNotice{}  

\begin{abstract}
Representing surfaces as zero level sets of neural networks recently emerged as a powerful modeling paradigm, named Implicit Neural Representations (INRs), serving numerous downstream applications in geometric deep learning and 3D vision. Training INRs previously required choosing between occupancy and distance function representation and different losses with unknown limit behavior and/or bias. In this paper we draw inspiration from the theory of phase transitions of fluids and suggest a loss for training INRs that learns a density function that converges to a proper occupancy function, while its log transform converges to a distance function. Furthermore, we analyze the limit minimizer of this loss showing it satisfies the reconstruction constraints and has minimal surface perimeter, a desirable inductive bias for surface reconstruction. Training INRs with this new loss leads to state-of-the-art reconstructions on a standard benchmark.  


\end{abstract}

\section{Introduction}
\emph{Implicit neural representation} (INR) refers to representing a surface in $\Real^3$, or more generally a hyper-surface in $\Real^d$, as a zero level-set of a neural network $f$ with parameters $\theta$:
\begin{equation}\label{e:S}
    \gS_\theta = \set{\vx\in\Real^d \ \vert \ f(\vx;\theta)=0}.
\end{equation}
Recent works have demonstrated the benefits of this representation \cite{mescheder2019occupancy,chen2019learning,Park_2019_CVPR,atzmon2019controlling}, including: the flexibility and expressive power of neural networks $f$, their favorable inductive bias / implicit regularization properties, and the facilitation of downstream learning applications involving 3D shapes. Indeed, implicit neural representations find more and more applications in machine learning, computer vision, and computer graphics \cite{saito2019pifu,niemeyer2019differentiable,niemeyer2019occupancy,gropp2020implicit,jiang2020local,sitzmann2020implicit,yariv2020multiview,tancik2020fourfeat}. 



\begin{figure}
    \centering
    \begin{tabular}{@{\hskip0pt}c@{\hskip1pt}c@{\hskip0pt}}
         \includegraphics[width=0.45\columnwidth]{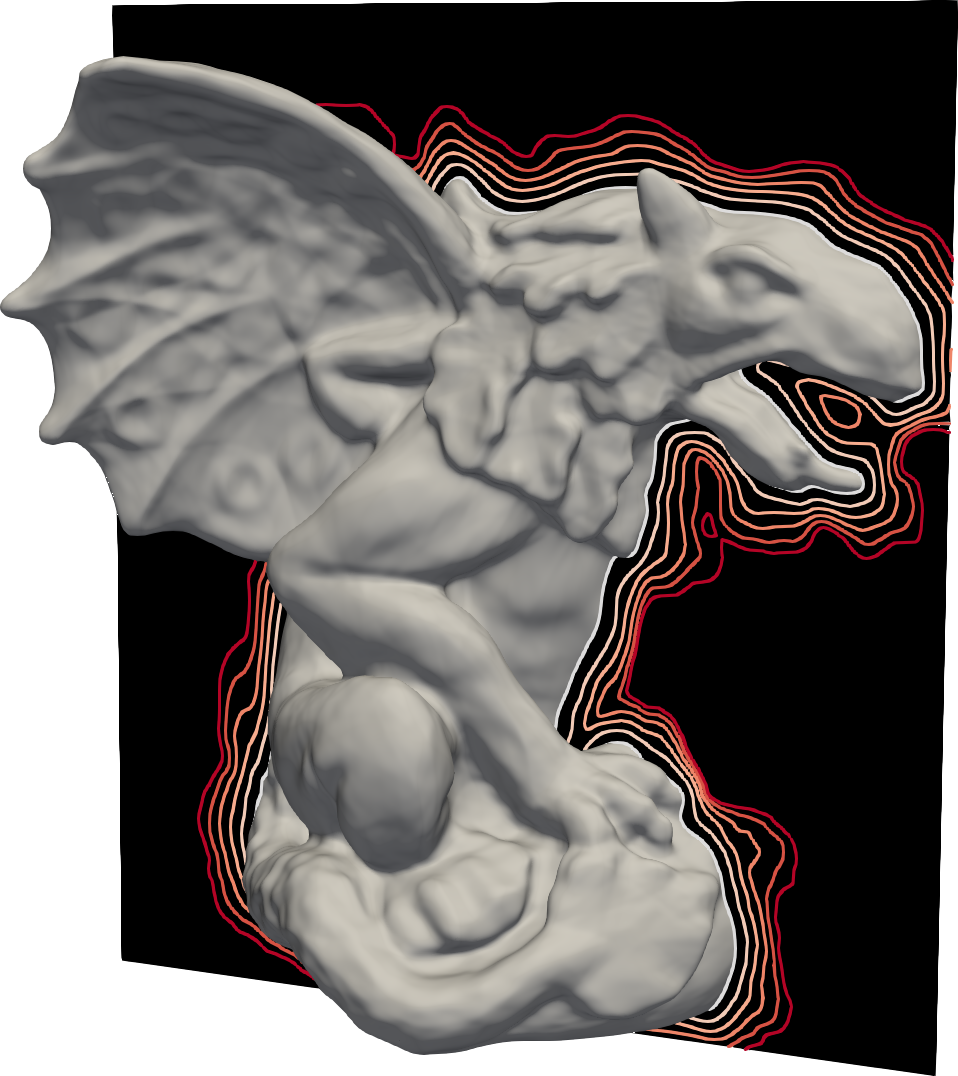}
    &  \includegraphics[width=0.55\columnwidth]{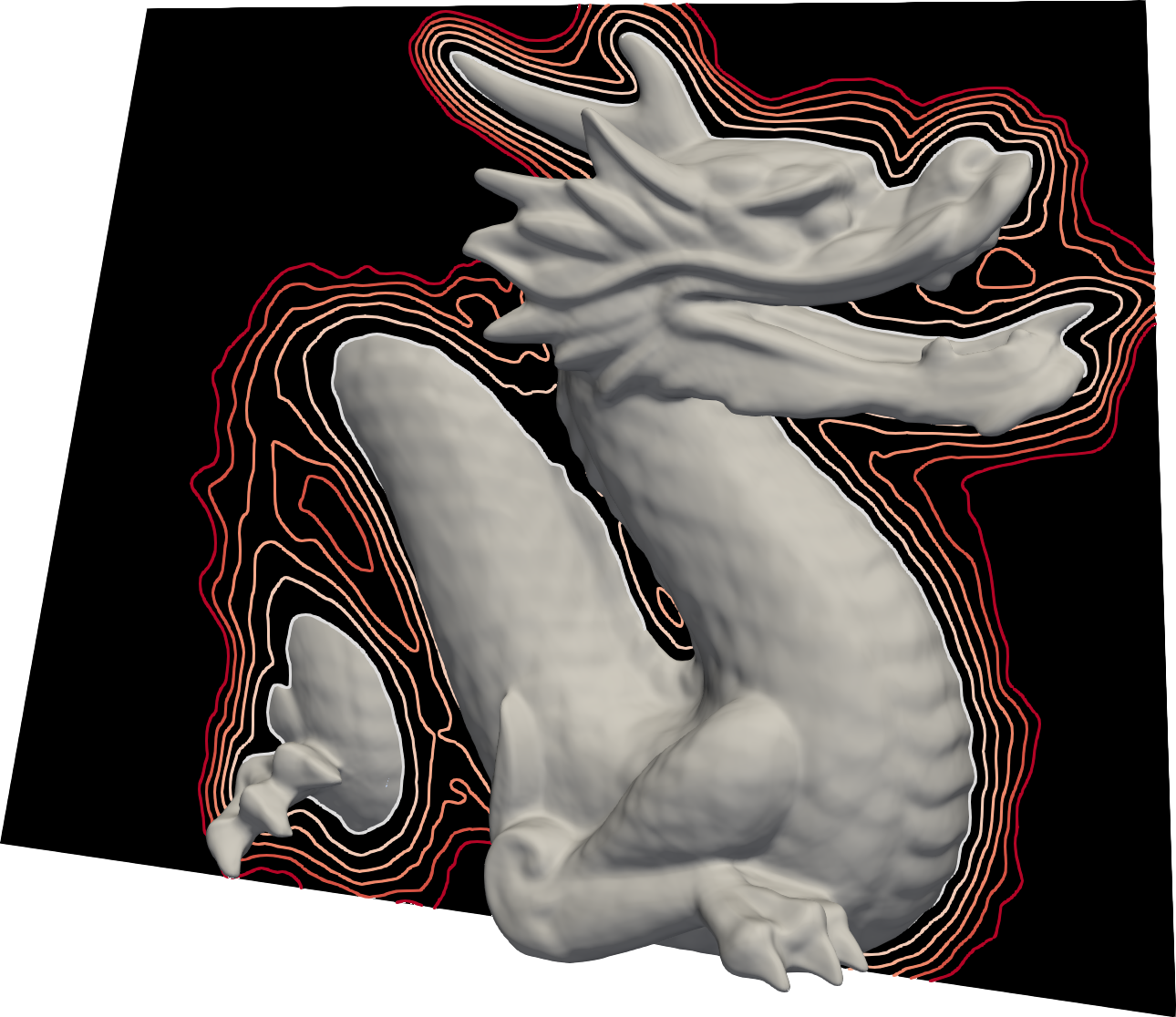}
    \end{tabular}
    \caption{Implicit neural representations learned with the novel PHASE loss. This new loss unifies occupancy and signed distance representations, and minimizes the perimeter of the  reconstructed surface. }
    \label{fig:teaser}
\end{figure}

Still, the problem of \emph{training} an implicit neural representation from raw data $\gX\subset \Real^d$ such as a point-cloud or a \emph{triangle soup}\footnote{A \emph{triangle soup} is a collection of triangles, not necessarily consistently oriented or defining a 2-manifold.} is not fully solved. 
%


Some previous works reduce the problem to a supervised regression problem \cite{Park_2019_CVPR,mescheder2019occupancy,chen2019learning,duan2020curriculum,chabra2020deep}. That is, compute an implicit representation to the geometry $\hat{f}$, and then, optimize the network $f$ to approximate $\hat{f}$. Other works consider the unsupervised version where only the raw input geometry is provided as input $\gX\subset \Real^d$ \cite{atzmon2019sal,gropp2020implicit,sitzmann2020implicit,tancik2020fourfeat,williams2020neural}. Two limitations of previous approaches are: First, the limit behavior and/or inductive bias of their loss/representation is often not clear, sometimes leading to undesired artifacts in the reconstructed geometry. Second, they require choosing between different implicit representations such as Signed Distance Functions (SDF) or occupancy functions. 

The goal of this paper is to design a new loss for training INRs from raw data, unifying the occupancy and SDF representations. This new loss, called PHASE, introduces a desirable inductive bias, and its limit behavior is theoretically understood. Motivated by phase transition theory of fluids we design a one parameter ($\eps$) loss that learns a signed density $u$, that converges (as $\eps\dtoo 0$) to a proper occupancy function, where its $\log$-transform, $w$, is a viscous (\ie, smoothed) signed distance function. Furthermore, we prove that the zero level-sets of the minimizers of this loss converge to a limit surface $\gS$ that is guaranteed to pass in vicinity of the input data $\gX$, and has minimal surface perimeter. These properties serve as a good inductive bias for the surface reconstruction problem, and we show that in practice this loss achieves superior reconstructions compared to relevant baseline methods. Furthermore, its minimal perimeter property makes it more suitable for high frequency INR training methods, such as \cite{tancik2020fourfeat,sitzmann2020implicit}, considerably reducing the ghost artifacts created when training with less structured losses. Figure \ref{fig:teaser} shows an example of two INR surfaces learned with this novel loss (PHASE).

To summarize our contributions, we introduce a novel loss for training INRs directly from input raw data that is supported by a well established limit theory and state-of-the-art reconstruction results. In the limit, the minimizers $u_0$ of the PHASE loss satisfy the following desirable properties: \emph{(i)} \emph{Proper occupancy}: $u_0\in \set{-1,1}$ almost everywhere, and $u_0= 0$ is a finite perimeter boundary surface $\gS$; \emph(ii) \emph{Zero reconstruction error}: The surface $\gS$ passes near the input data $\gX$; \emph{(iii)} \emph{Minimal perimeter}: $\gS$ minimizes surface perimeter; and \emph{(iv)} \emph{Signed distance function}: The log transform of $u_0$ is a smoothed signed distance function to $\gS$.

\section{Preliminaries}\label{s:prelims}
In this section we provide some preliminaries and background for the main constructions and proofs in the next section. We work in $\Real^d$, and $\norm{\vx}$ will denote the standard euclidean norm of a vector $\vx\in\Real^d$. A subset $\Omega\subset \Real^d$ is called a \emph{Lipschitz domain} if it is open, bounded, connected, and has a boundary that is the union of finitely many Lipschitz manifolds; $|\Omega|=\int_\Omega 1$ is the total volume of $\Omega$. $L^p(\Omega)$ is the space of scalar functions, $f:\Omega\too\Real$, with finite Lebesgue integral $\int_\Omega |f|^p < \infty$; and $C^\infty_c(\Omega)$ denotes the space of infinitely smooth functions with a compact support in $\Omega$. $L^p(\Omega;\Real^d)$, $C^\infty_c(\Omega;\Real^d)$ will denote vector valued functions $\vv:\Omega\too\Real^d$ with each coordinate in $L^{p}(\Omega)$, $C^\infty_c(\Omega)$, respectively. 

\subsection{Sobolev spaces}\label{ss:sobolev}
Sobolev spaces generalize the classical notion of differentiable functions to \emph{weakly} differentiable functions. They are useful in the field of partial differential equations and calculus of variations since, in contrast to classical differentiable functions, they survive the process of limit. $W^{1,p}(\Omega)$, $p\in [1,\infty]$, is the Sobolev space of all $L^p(\Omega)$ functions with first \emph{weak derivatives} also in $L^p(\Omega)$ (definition of weak derivatives and more details on Sobolev spaces can be found in \cite{adams2003sobolev,gilbarg2015elliptic,leoni2017first}). One interesting relation to neural networks is the following lemma proved in the supplementary:
\begin{lemma}\label{lem:W1}
Neural networks with ReLU activations are in $W^{1,p}(\Omega)$, for all $p\in [1,\infty]$, if $\Omega\subset\Real^d$ is a Lipschitz domain. 
\end{lemma}


\subsection{Functions of bounded variation and perimeter} \label{ss:BV_perimeter}
The space of bounded variation functions over $\Omega$, denoted $\BV(\Omega)$, is defined as the collection of all functions $f\in L^1(\Omega)$ such that 
\begin{equation}\label{e:def_BV}
    \int_\Omega \norm{\nabla f} \eqdef \sup_{\substack{\vv\in C^\infty_c(\Omega;\Real^d)\\ \norm{\vv}\leq 1}}\int_{\Omega} f \divv(\vv) < \infty
\end{equation}
where $\norm{\vv}\leq 1$ means that for all $\vx\in\Omega$, $\norm{\vv(\vx)}\leq 1$, and $\divv$ is the divergence operator. $\BV(\Omega;\gA)$ is the set of bounded variation functions over $\Omega$ with image in the set $\gA$.

\emph{Surface perimeter} is a generalization of surface area to boundaries of general sets: Given a subset $\gI\subset \Omega$, its perimeter in $\Omega$ is defined by 
\begin{equation}\label{e:per}
    \per_\Omega(\gI) \eqdef \int_\Omega \norm{\nabla \one_{_\gI}}
\end{equation}
where $\one_{_\gI}$ is the indicator function of the set $\gI$. If $\gI$ has a smooth boundary then it can be shown with the help of the Divergence Theorem that $\per_\Omega(\gI)$ coincides with the standard area of $\partial\gI \cap \Omega$, \ie, $\gH^{d-1}(\partial\gI\cap\Omega)$, the $d-1$ dimensional Hausdorff measure \cite{rindler2018calculus}. 

\subsection{$\Gamma$-convergence}\label{ss:gamma_convergence}
$\Gamma$-convergence is a tool for analyzing the limit behavor of a family of functionals. In our case we consider such a family, parameterized by $\eps$: $\gF_\eps:L^1(\Omega)\too\Real$, and a ``limit'' functional $\gF_0:L^1(\Omega)\too\Real$. We say that $\gF_\eps$ $\Gamma$-converge to $\gF_0$ (as $\eps\dtoo 0$) in the $L^1(\Omega)$ norm, if the following two conditions hold:
\begin{enumerate}
    \item \emph{$\liminf$}. For all sequences $u_\eps\too u$, $$\gF_0(u)\leq \liminf_{\eps\dtoo 0}\gF_\eps(u_\eps).$$
    \item \emph{Recovery sequence}. For every $v\in L^1(\Omega)$ there exists a sequence $v_\eps\too v$ so that $\lim_{\eps \dtoo 0}\gF_\eps(v_\eps)=\gF_0(v)$. 
\end{enumerate}
The convergence of $u_\eps, v_\eps$ in the above conditions is in the $L^1(\Omega)$ sense. The most important property of $\Gamma$-convergence is that it implies convergence of minimizers and minimal values \cite{modica1987gradient,rindler2018calculus}:
\begin{theorem}\label{thm:gamma_conv_of_mins}
If $u_\eps\too u_0$, where $u_\eps$ is a minimizer of $\gF_\eps$, then $\Gamma$-convergence $\gF_\eps\too\gF_0$ would imply $u_0$ is a minimizer of $\gF_0$ and $\lim_{\eps\dtoo 0}\gF_\eps(u_\eps)=\gF_0(u_0)$.   
\end{theorem}
Furthermore, if the functional family $\gF$ is also \emph{coercive}, meaning $\gF_\eps\leq c <\infty$, for all $\eps>0$, is compact, the convergence (up-to a subsequence) of $u_\eps$ is established.

\paragraph{Local minimizers.} While $\Gamma$-convergence deals with convergence of global minimizers, in case the limit functional $\gF_0$ has a strict local minimum $u_0$, the recovery sequence property shows that there exist a sequence of asymptotically local minimizers of $\gF_\eps$ that converge to $u_0$. In fact, there exists a recovery sequence of (perfect) local minimizers of $\gF_\eps$ converging to $u_0$, see Theorem 4.1 in \cite{kohn1989local}.

\begin{figure}[t]
    \centering
    \begin{tabular}{cc}
        \includegraphics[width=0.4\columnwidth]{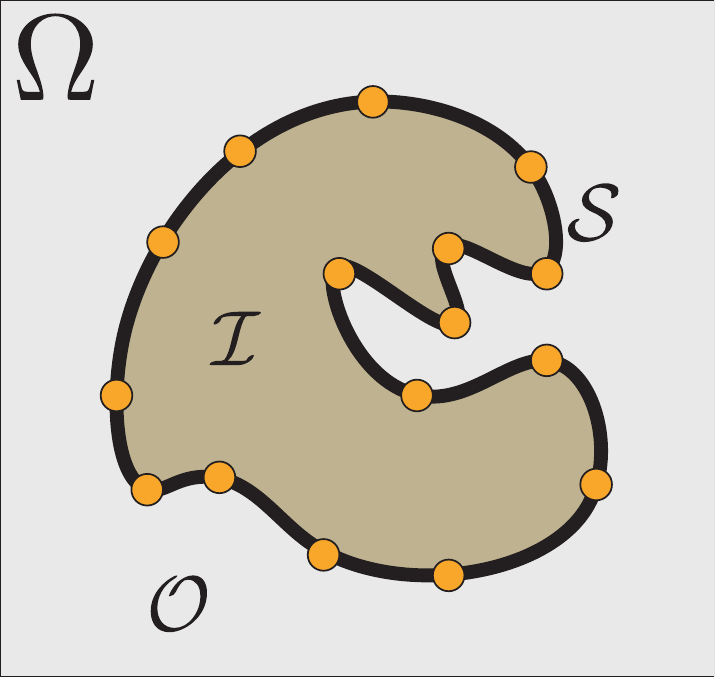} & \includegraphics[width=0.4\columnwidth]{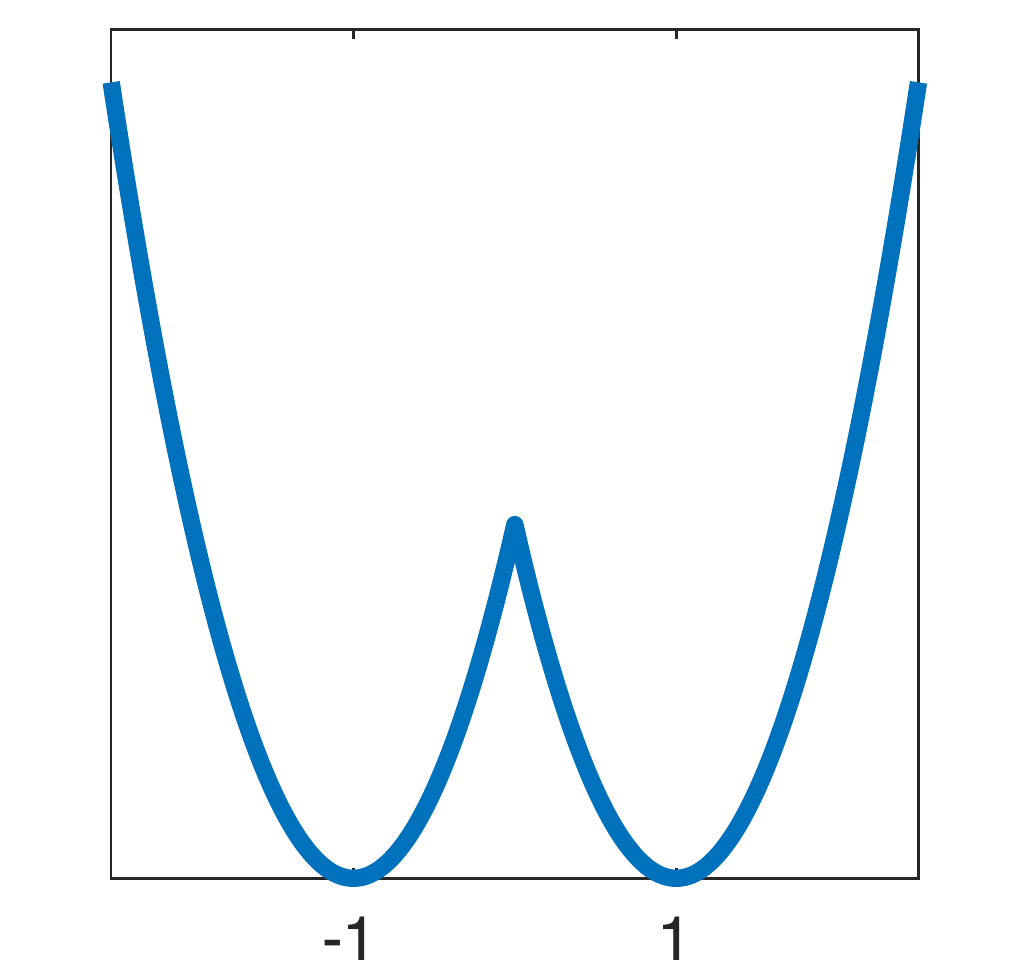}  \\
        (a) & (b) 
    \end{tabular}
    \caption{(a) The surface reconstruction problem asks to decompose a container ($\Omega$) to interior ($\gI$) and exterior ($\gO$) with $\gX$ (centers of orange disks) being close to the interface $\gS$ (in black); (b) a visualization of a double-well potential $W$. }
    \label{fig:setup}
\end{figure}

\section{Method}
We are considering the problem of learning an implicit neural representation (INR), \ie,  \eqref{e:S}, from an input raw geometry $\gX\subset\Omega$, where $\Omega$ is some Lipschitz domain (\eg, think of a bounding box of $\gX$). A prototypical example of raw geometry is a point cloud $\gX=\set{\vx_i}_{i\in I}$, where $\vx_i\in\Real^d$.

We are drawing inspiration from the phase transition problem of fluids  \cite{modica1987gradient,sternberg1988effect,gurtin1989phase,hutchinson2000convergence,rindler2018calculus}, where two fluids find equilibrium in a container ($\Omega$) by minimizing a certain energy. In our case, the two fluids are chosen to represent the \emph{interior} and the \emph{exterior} of the shape. Mathematically, this is modeled with a function $u:\Omega\too [-1,1]$. An $x\in\Omega$ is deemed interior if $u(\vx)=-1$, exterior if $u(\vx)=1$, and $u(\vx)\in(-1,1)$ represents some ``mixture'' of interior and exterior. We will denote by 
\begin{equation}\label{e:I_and_O}
\begin{aligned}
\gI &= \set{\vx\in\Omega\ \vert \ u(\vx)\in[-1,0)}, \\ 
\gO &= \set{\vx\in\Omega\ \vert \ u(\vx)\in(0,1]}, \\
\gS &= \partial \gI \cap \Omega
\end{aligned}
\end{equation} the (rounded) interior of the shape, its exterior, and the boundary surface of the shape, respectively. 
$u$ can be seen as an occupancy function \cite{mescheder2019occupancy,chen2019learning}; note that composing $u$ with a linear transformation taking $[-1,1]$ to $[0,1]$ will provide the more standard (yet equivalent) form of occupancy function with image in $[0,1]$.  Differently from the above mentioned works, our goal here is to learn $u$ directly from the raw data $\gX$. 

In phase transition of fluids one looks for an equilibrium density $u$ minimizing the energy 
\begin{equation}\label{e:int_W}
    \gE(u)=\int_\Omega W(u),
\end{equation}
over all $u\in L^1(\Omega)$ satisfying the "total mass" constraint: 
\begin{equation}\label{e:int_m}
  \int_\Omega u = m \in \big(-|\Omega|,|\Omega|\big).  
\end{equation}
The function $W:\Real\too\Real$ is chosen to be a \emph{double-well potential}, \ie, it is non-negative and has zeros only at $\set{-1,1}$, \eg, Figure \ref{fig:setup}(b). As we will see, with this choice of $W$, the energy $\gE$ will drive (almost) every point $\vx\in\Omega$ to choose a side: $u(\vx)=-1$ (corresponding to the interior) or $u(\vx)=1$ (corresponding to the exterior). The global mass constraint in \eqref{e:int_m} is motivated from the physical property that the fluids maintain their mass in a closed container $\Omega$. This constraint, however, is not suitable for surface reconstruction since we don't know the mass/volume of the shape we are looking for, and enforcing hard global constraints is challenging in deep networks, the main work horse of INRs. 

We therefore start by considering the following (unconstrained) energy functional:
\begin{equation}\label{e:F}
    \gF(u)=\lambda\gL(u) + \begin{cases}
    \int_\Omega W(u)   & u\in L^1(\Omega) \\
    +\infty & \text{otherwise},
    \end{cases}
\end{equation}
where $\lambda>0$ is a parameter, and $\gL$ is the \emph{reconstruction loss}, responsible for encouraging $u$ to vanish in vicinity of $\gX$: 
\begin{equation}\label{e:L}
   \gL(u) =  \E_{\vx\sim \gX} \abs{\frac{1}{|B_\vx|}\int_{B_\vx}u},  
\end{equation} 
where $B_\vx$ is some choice of a small radius ball centered at $\vx$, and $\vx\sim \gX$ means that $\vx$ is sampled from some distribution over $\gX$. For example, if $\gX$ is a point cloud, a natural choice would be to give every point an equal probability. We note that the following analysis also works for a more general reconstruction error: the integral in \eqref{e:L} can be replaced with $\int_\Omega u \varphi_\vx$, where $\varphi_\vx:\Omega\too \Real_+$ is an integrable and essentially bounded in $\vx,\vy$ over $\Omega$.

The energy $\gL$ can be seen as a local and soft version of the total density constraint in \eqref{e:int_m}. It carries an interesting benefit that will be shown later on: in the context of the functional $\gF$ it will force $u$ to change sign in vicinity of $\gX$, rather than just zero out, which is a useful property for INRs. 


It is not hard to see, however, that without further regularization the functional in \eqref{e:F} is too permissive, allowing far too many solutions (\ie, global minima):
\begin{lemma}
Let $\gX=\set{\vx_i}_{i\in I}$ be a (finite) point cloud, and assume the balls $B$ are sufficiently small. Then, almost every arbitrary decomposition of the container $\Omega=\gI\cup\gO$ so that $\gX\subset \partial \gI \cap \Omega$ provides a global minimum to $\gF$.
\end{lemma}
\begin{proof}
Assume the balls $B$ are sufficiently small so $B(\vx_i)$, $i\in I$, are all mutually disjoint. Next, consider an arbitrary decomposition $\Omega=\gI\cup\gO$, where the boundary $\partial\gI$ is smooth, and each $B(\vx_i)\cap \partial \gI$ is exactly half a ball (\ie, congruent to $B\cap \text{half-space of }\Real^d)$. Clearly, there is an infinite number of such decompositions. Lastly, let $u(\vx)=\one_{_\gO}(\vx)-\one_{_\gI}(\vx)$, where as above $\one$ denotes the indicator function of a set. Then, $\gF(u)=0$. 
\end{proof}

\begin{figure}
    \centering
    \includegraphics[width=\columnwidth]{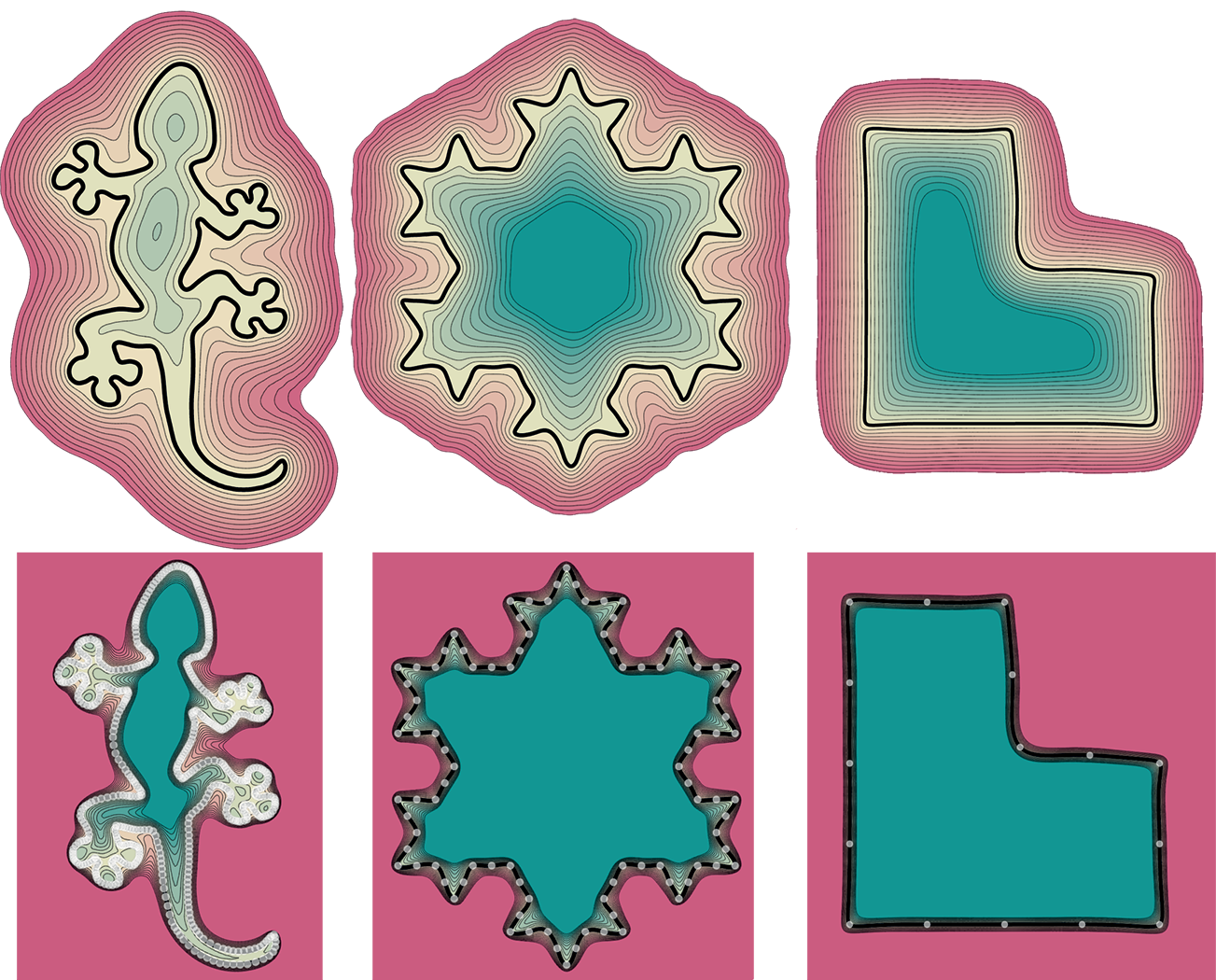}
    \caption{Learned densities $u_\eps$ (bottom row) for input point cloud $\gX$ (depicted as white dots). The log transform of the densities, $w_\eps$, is visualized in the top row, with zero level set in bold. Negative values are colored in green, and positive in red. }
    \label{fig:u_ns_w}
\end{figure}

\subsection{Regularization and minimal perimeter solutions} In phase transition theory the functional in \eqref{e:int_W} is further refined to enforce some regularity on the minimizers $u$ (and consequently on the interface surface $\gS$) by adding a small perturbation of the form $\epsilon\norm{\nabla u}^2$. This is called the Van der Waals-Cahn-Hilliard (WCH) theory of phase
transitions \cite{modica1987gradient,sternberg1988effect}. We adapt it for surface reconstruction: 
\begin{equation}\label{e:F_eps}
\gF_{\epsilon}(u)=\lambda\gL(u)+\begin{cases} \int_\Omega \eps\norm{\nabla u}^2 + W(u) & u\in W^{1,2}(\Omega)\\ 
    +\infty & \text{otherwise}
    \end{cases}
%
\end{equation}
and $\epsilon>0$ is a regularization parameter, and we switched to $W^{1,2}(\Omega)$ Sobolev function space to allow integration of the squared gradient norm; see Section \ref{ss:sobolev} (and reference therein) for a short introduction on Sobolev spaces. Note that we added dependence on $\epsilon$ in the functional notation; in our analysis we will consider $\lambda$ to be a function of $\eps$, \ie,  $\lambda=\lambda(\eps)$. 

As it turns out, the limit behavior of the minimizers of $\gF_\eps$, denoted hereafter $u_\eps$, can be characterized pretty well, \emph{regardless} of the specific choice of $W$: as $\eps\dtoo 0$ they converge to a limit solution $u_0$ with the following three key properties: 
\begin{enumerate}[label=(\emph{\roman*})]
    \item \emph{Proper occupancy}: $u_0\in \BV(\Omega;\set{-1,1})$.
    \item \emph{Zero reconstruction error}:  $\gL(u_0)=0$. 
    \item  \emph{Minimal perimeter}:   $u_0$ minimizes $\per_\Omega(\gI)$. 
\end{enumerate}
Several comments and clarifications are in order. 
First, (i) implies that $u_0$ takes values in $\set{-1,1}$ almost everywhere in $\Omega$, meaning it leaves no ``mixed"/undetermined areas of positive measure; in other words, almost every point is either perfectly inside or outside the shape. Mathematically this means that $u_0$ can be written as 
\begin{equation}\label{e:u_0}
u_0(\vx) = -\one_\gI(\vx) + \one_{\Omega\setminus \gI}(\vx),
\end{equation}
where as above $\one$ is the indicator function of the relevant set, $\gI$ is defined in \eqref{e:I_and_O} with $u=u_0$ (and in this case coincides up to measure zero with $u_0^{-1}(-1)$), and the equality should be interpreted up to a set of measure zero. 
Second, property (ii), $\gL(u_0)=0$, means that $u_0$ satisfies the constraints $\int_{B_\vx}u_0=0$ exactly, and in view of (i) indeed has to change sign in each $B_\vx$. 
Third, since $u_0\in\BV(\Omega;\set{-1,1})$, the set $\gI$ has a bounded surface perimeter in $\Omega$ (see Section \ref{ss:BV_perimeter} and the supplementary). (iii) asserts that the $\gI$ will have minimal \emph{surface perimeter}. 
Lastly, it can be shown that the set $\gF_\eps\leq c<\infty$ is compact in $L^1(\Omega)$ (see \eg, Section 3.1 in \cite{kohn1989local}), therefore convergence of $u_\eps$ is up to a sub-sequence. 

The three statements (i), (ii) and (iii) above are shown using the tool of $\Gamma$-convergence as introduced in Section \ref{s:prelims}. We define our target (limit) functional to be measuring the boundary surface perimeter for solutions that have zero reconstruction loss:
\begin{equation}\label{e:F_0}
    \gF_0(u)=  \begin{cases} \sigma_0 \per_\Omega(\gI) & u\in \BV(\Omega,\set{-1,1}), \\&  \text{and } \gL(u)=0 \\ 
    +\infty & \text{otherwise}
    \end{cases}
\end{equation}
where $\sigma_0$ is a constant, and $\gI$ defined in \eqref{e:I_and_O}. We prove:
\begin{theorem}\label{thm:main_gamma}
    If $\lambda=\lambda(\eps)$ is chosen so that $\lambda\eps^{-\frac{1}{2}}\too \infty$ while $\lambda\eps^{-1/4}\too 0$, then $\eps^{-\frac{1}{2}}\gF_\eps$ $\Gamma$-converge to $\gF_0$.
\end{theorem}
The conditions on $\lambda$ ask it to go to zero with $\eps$ but at certain slower asymptotic rate; intuitively, letting $\lambda=\eps^\alpha$, then the conditions boil to $\alpha\in(1/4,1/2)$. This is only a sufficient condition but gives a practical rough estimate for $\lambda$ given $\eps$ in \eqref{e:F_eps}. 

\paragraph{Intuition and proof idea.}
We next provide some intuition regarding the connection of $\gF_\eps$ to the surface perimeter $\per_\Omega(\gI)$, followed by the proof idea (given in full in the supplementary). 

The relation of $\gF_\eps$ and $\per_\Omega(\gI)$ can be understood using an auxiliary function introduced in the proof of Theorem 13.6 \cite{rindler2018calculus}: Let $H(t) = 2\int_{-1}^t \sqrt{W(s)} ds$ and compute the variation of the function $H(u(\vx))$:
\begin{align*}\label{e:nabla_H}
\int_\Omega \norm{\nabla H(u)}  &= \int_\Omega 2\sqrt{W(u)} \norm{\nabla u}  \\ &\leq \int_\Omega \sqrt{\eps}\norm{\nabla u}^2 +   \frac{1}{\sqrt{\eps}}W(u) ,
\end{align*}
where the equality uses the chain rule, and the inequality uses the standard (Young) inequality $2ab\leq a^2+b^2$.

Now, if properties (i) and (ii) hold, namely $u$ converges to $u_0$, which is a proper occupancy function (see \eqref{e:u_0}) with zero reconstruction error $\gL(u_0)=0$, then we have based on \eqref{e:def_BV}:
\begin{align*}
\int_\Omega \norm{\nabla H(u)} &\approx 
\int_\Omega \norm{\nabla H(u_0)} = \sup_{\norm{\vv}\leq 1}\int_\Omega H(u_0) \text{div}(\vv) \\
&= H(1)\int_\Omega \norm{\nabla \one_{\Omega\setminus \gI}} = H(1) \per_\Omega(\gI),
\end{align*}
where the third equality is due to the fact that $H(u_0) = H(-1)\one_\gI + H(1)\one_{\Omega\setminus \gI}=H(1)\one_{\Omega\setminus \gI}$, since $H(-1)=0$. The constant $\sigma_0$ from \eqref{e:F_0} is now given explicitly by $\sigma_0=2\int_{-1}^1 \sqrt{W}$. Lastly, since $\gL(u)\approx\gL(u_0)=0$ we have \begin{equation}\label{e:F_eps_geq_per_omega}
  \eps^{-\frac{1}{2}}\gF_\eps(u)\gtrapprox  \sigma_0 \per_\Omega(\gI).
\end{equation}

We proceed with the proof idea. The following is based on $\Gamma$-convergence results of the WCH theory, \eg,  \cite{modica1987gradient,sternberg1988effect,fonseca1988gradient}, and in particular the proof technique in \cite{rindler2018calculus}. The main difference in our setting compared to that of \cite{rindler2018calculus} is the addition of the surface reconstruction loss $\gL$. To prove $\Gamma$-convergence one must show that the two conditions formulated at Section \ref{ss:gamma_convergence}, namely, the $\liminf$ and recovery sequence, hold. 

The $\liminf$ is the relatively easier part of the WCH theory and resembles the intuition provided above (\eqref{e:F_eps_geq_per_omega}).  We show that for an $L^1(\Omega)$ convergent sequence $u_\eps\too u$, $\liminf_{\eps\dtoo 0}\gF_\eps(u_\eps)\geq \gF_0(u)$. The main observation here is that the rate of $\lambda=\lambda(\eps)$ forces the reconstruction loss to vanish in the limit, that is $\gL(u)=0$, showing property (ii) (property (i) comes out of the standard part of the WCH theory). After $\gL(u)=0$ is shown the $\liminf$ property comes out from existing arguments in the WCH theory.

The second part is establishing the recovery sequence, namely for every $u\in L^1(\Omega)$ find a converging sequence in $L^1(\Omega)$, $u_\eps\too u$ where $\gF_\eps(u_\eps)\too \gF_0(u)$. The main observation here is that we can utilize the recovery sequence constructed in the WCH theory and show that it is also a recovery sequence for the reconstruction functionals. The idea is to show that if the reconstruction part satisfies $\lambda(\eps) \gL(u_\eps)\too 0$ then it does not interfere with the lim equality $\gF_\eps(u_\eps)\too \gF_0(u)$. This what leads to the asymptotic balance of $\lambda$ required in the theorem formulation. \\

\textbf{Local minimizers.} Lastly, note that, as mentioned in Section \ref{ss:gamma_convergence}, $\Gamma$-convergence would imply convergence of local minimizers to strict local minima of $\gF_0$. 


\subsection{Signed distance functions}
In this section we make use of the degree of freedom in choosing the double well potential $W$ in \eqref{e:F_eps}, and show that a particular choice, together with a simple transformation of the minimizers $u_\eps$ provide the fourth property:

\quad \emph{(iv)}  \emph{Signed distance function}.

In more detail, the signed distance function to the surface $\gS$ enclosing a shape $\gI$ is 
\begin{equation}\label{e:sdf}
 d_\gS(\vx)=\parr{-\one_{\gI}(\vx) + \one_{\Omega\setminus\gI}(\vx)}  \min_{\vy\in\gS}\norm{\vx-\vy}.  
\end{equation}
We advocate the following choice for the potential:
\begin{equation}
    W(s)=s^2-2|s|+1.
\end{equation}
The graph of $W$ is shown in Figure \ref{fig:setup}(b). 
Given a minimizer $u_\eps$ of \eqref{e:F_eps} we define our final INR by
\begin{equation}\label{e:w}
w_\eps(\vx) = -\sqrt{\eps}\log\parr{1-\abs{u_\eps(\vx)}}\sign(u_\eps(\vx)).
\end{equation}

We claim that $w_\eps$ will approximate $d_\gS$. We first show:
\begin{theorem}\label{thm:laplace_u}
Let $u_\eps\in W^{1,2}(\Omega)$ be a (local) minimizer of $\gF_\eps$, and $O\subset\Omega \setminus \cup_{\vx\in\gX} B_\vx$ a domain where $u_\eps\ne 0$. Then, $u_\eps$ is smooth in the classical sense in $O$ and satisfies 
\begin{equation}\label{e:linear}
-\eps\Delta u_\eps + u_\eps -\sign(u_\eps) = 0.
\end{equation}
\end{theorem}
This theorem is proved by first applying the Euler-Lagrange conditions (see \eg, Theorem 3.1 in \cite{rindler2018calculus}) for a minimizer $u$ of $\gF_\eps$ in $O$, and noting that $W'(s)=2s-2\sign(s)$, which is differentiable in $O$. In particular, \eqref{e:linear} will be satisfied in the weak sense with any test function $\psi\in C_c^\infty(O)$.
%
Second, regularity results for elliptic operators (\eg, Corollary 8.11 in \cite{gilbarg2015elliptic}) show that $u$ is smooth in $O$ and satisfies \eqref{e:linear} in the classical sense.  See the full details in the supplementary. 

Theorem \ref{thm:laplace_u} can be used to show that $w_\eps$ (as defined in \eqref{e:w}) is a solution to the so-called \emph{viscosity Eikonal equation}, whose solutions are basically a smoothed version of signed distance functions. 
\begin{theorem}\label{thm:visc_eikonal}
    Let $O\subset\Omega$ be a domain as defined in Theorem \ref{thm:laplace_u}. Then, over $O$, $w_\eps$ satisfies
    \begin{equation}\label{e:w_eps_viscosity_eikonal}
    -\sqrt{\eps}\Delta w_\eps + \sign(u_\eps) (\norm{\nabla w_\eps}^2-1) = 0
    \end{equation}
\end{theorem}
The proof of this theorem is given in the supplementary material; the proof follows a direct computation. Note that the idea behind defining $w_\eps$ is the so-called Cole-Hopf transformation, known for transforming some eikonal-type non-linear equations to linear ones, see Section 4.4.1 in \cite{evans1998partial}, and \cite{schieborn2006viscosity,gurumoorthy2009schrodinger,sethi2012schrodinger,belyaev2015variational}.

Theorem \ref{thm:visc_eikonal} asserts that $w_\eps$ solves the viscosity Eikonal equation and it is expected, under certain conditions, to reproduce the signed distance function via the vanishing viscosity method \cite{crandall1983viscosity,schieborn2006viscosity}. We further show a closely related result by \citet{varadhan1967behavior} allowing to prove that on some fixed domain $O$, $w_\eps$ defined by equations \ref{e:w} and \ref{e:linear} converge to the signed distance function: 
\begin{theorem}\label{thm:var}
    Let $O$ be an open set, and $u_\eps$ a solution to \eqref{e:linear} in $O$, $u_\eps=0$ on $\partial O$ and $u_\eps\ne 0$ in $O$. Then $ w_\eps \too  \sign(u_\eps) d_{\partial O}$ pointwise uniformly in any compact subset $O\cup\partial O$. 
\end{theorem}
Note that $\sign(u_\eps)$ is well defined in $O$ since we assume that it does not vanish in $O$. 
\begin{proof}
First, assume $\sign(u_\eps)>0$ in $O$. Then, $u_\eps$ satisfies:
\begin{equation*}
\begin{aligned}
-\eps\Delta u_\eps + u_\eps - 1 = 0 & \qquad \text{in } O\\
u_\eps = 0 & \qquad \text{in } \partial O
\end{aligned}
\end{equation*}
Now the change of variables $v_\eps=1-u_\eps$ leads to 
\begin{equation}\label{e:var_linear}
\begin{aligned}
\frac{1}{2}\Delta v_\eps  = \frac{1}{2\eps}v_\eps  & \qquad \text{in } \Omega\\
v_\eps = 1 & \qquad \text{in } \partial\Omega
\end{aligned}
\end{equation}
Therefore, Theorem 2.3 in \cite{varadhan1967behavior} with $\lambda=\frac{1}{2\eps}$ now implies 
$$-\sqrt{\eps}\log(v_\eps)=-\sqrt{\eps}\log(1-u_\eps)\xrightarrow{\eps\too 0} d_{\partial O}$$
uniformly in compact subsets of $O\cup\partial O$. The second case where $\sign(u_\eps)<0$ in $O$ is proved similarly, and provided in the supplementary. \vspace{-10pt}
\end{proof}
Using Theorem \ref{thm:var} to provide a rigorous proof of convergence of $w_\eps$ requires two extra things: first, deal with the reconstructions constraints, \eqref{e:L}, that inject another perturbation in vicinity of data points (but can be made arbitrary small), and establishing a version of Theorem \ref{thm:var} with slightly moving boundaries $O_\eps$. We leave these further analyses to future work. 
 


\subsection{Additional losses}
Theorem \ref{thm:visc_eikonal} hints of the possibility that $w_\eps$ does not satisfy the viscosity Eikonal equation near the input geometry $\gX$.  Indeed, this is verified in the experiments section. Therefore, although not included in our analysis (and marked as an interesting future work), we found that empirically incorporating losses constraining the gradients of $w_\eps$ to unit length at $\gX$ is useful for INR training, especially in 3D. There are two use cases: If given as input a normal field over $\gX$, namely $\vn:\gX\too \gS(\Real^d)$, where $\gS(\Real^d)$ is the unit sphere, then the loss \begin{equation}\label{e:N_intr}
    \gN_{\text{intr}}(u) = \E_{\vx\sim \gX} \norm{\vn(\vx) -\nabla w(\vx) }^p
\end{equation}
encourages the normal to $\gS$ to be close to $\vn$ over $\gX$. Note that for $\vx\in\gS$, \eqref{e:w} implies that $\nabla w(\vx)=\sqrt{\eps}u(\vx)$. In our implementation we used $p=1$. 

In the case the normals at $\gX$ are not known we incorporate the unit gradient constraint over $\gX$, similar to \cite{gropp2020implicit}:
\begin{equation}\label{e:N_unit}
    \gN_{\text{unit}}(u) = \E_{\vx\sim \gX} \abs{1 - \norm{\nabla w(\vx)} }^p,
\end{equation}
where in our implementation we used $p=2$. 

Our final loss is 
\begin{equation}\label{e:loss}
\text{loss}(u) = \gF_\eps+\mu \gN_{\star},\vspace{-5pt}
\end{equation}
where $\gN_\star$ is defined either by \eqref{e:N_intr} or \eqref{e:N_unit}, and $\mu\geq 0$ is a hyperparameter. In total we have three hyper-parameters $\eps,\lambda,\mu$; we detail below how we set those.

\subsection{Implementation details}\label{s:implementation_details}


\paragraph{Network architecture and initialization.}
We have used the same architecture for $u$ as previous work \cite{Park_2019_CVPR,gropp2020implicit}, namely $u:\Real^d\too\Real$ is a multilayer perceptron (MLP) with 8 layer of 512 neuron each, and a single skip connection concatenating the input vector to the input to the 4-th layer. We used either ReLU or Softplus activation with $\beta=100$, and the geometric initialization of \cite{atzmon2019sal} initializing the network to be an approximated SDF to the $d$-dimensional unit sphere. We normalized $\gX$ to have unit max norm, and took $\Omega$ to be a scaled version of the axis-aligned bounding box of $\gX$; we use as scale $1.5$ in $d=3$, and  $2$ in $d=2$. 

In our loss, \eqref{e:loss}, we have three hyper-parameters $\eps,\lambda,\mu$. In all the experiments we used $\eps=0.01$, and did a parameter search over $\lambda,\mu \in [0.2,20]$ (note that $\eps^{1/3}\approx 0.2$, and $1/3$ is in the range asserted by Theorem \ref{thm:main_gamma}), including also $\mu=0$. 
The integrals in the loss are stochastically estimated by sampling $\vx\sim \gU(\Omega)$, the uniform distribution over $\Omega$, for the WHC integral, and $\vx\sim \gN(\vx,\sigma^2)$, the normal distribution centered at $\vx$ with standard deviation $\sigma\in\set{10^{-3},10^{-4}}$ for $\mathcal{L}$.  In some experiments we used Fourier features \cite{tancik2020fourfeat} as a first fixed layer in the network, $\delta_k(\vx)\in\Real^{2kd}$, where $k$ is representing the number of frequencies used. See the  supplementary for more implementation details. \vspace{-8pt}

%

\section{Previous work}\vspace{-5pt}

Representing shapes as zero level sets of neural networks has been popularized recently \cite{Park_2019_CVPR,mescheder2019occupancy,chen2019learning,atzmon2019controlling}, and since been used for many applications in 3D vision and graphics. One standing challenge, tackled in this paper is learning INRs from raw geometric data. Similar problem was addressed in \cite{atzmon2019sal,gropp2020implicit,sitzmann2020implicit,tancik2020fourfeat,atzmon2019sal,williams2020neural}. In IGR \cite{gropp2020implicit} the authors use the implicit regularization property of neural network training to fit a smooth SDF to an input point cloud with or without normals. \citet{williams2020neural} use the limit kernel of infinitely wide shallow networks to directly solve a kernel regression problem to fit an implicit representation to a point cloud with normals. High frequency signal learning techniques show that adding an input Fourier feature layer \cite{tancik2020fourfeat}, or using periodic activation functions \cite{sitzmann2020implicit} facilitate faster learning of high frequency details. 
In this group of works the minimizers of the loss are generally not well understood and indeed, as we show in the experiments, sometimes lead to undesirable reconstructed parts. The minimal perimeter property seems to add an important inductive bias.

In the signed agnostic learning (SAL) method \cite{atzmon2019sal,atzmon2020sald} unsigned distance is fitted by a sign-agnostic regression to introduce a signed local minimum. In fact, SAL can be seen as a version of the unregularized loss in \eqref{e:F}. Indeed, the loss $\abs{|a|-b}^\ell$ (see equation 5 in \cite{atzmon2019sal}) is a double well potential as a function of $a$, similar to $W$ in our formulation, with the difference of spatially changing $W$. In this sense our work suggests that adding a gradient regularization to SAL could be a useful idea. Generalizing our analysis to spatially varying potential, could also be interesting future research.



\begin{figure}
    \centering
    \includegraphics[width=\columnwidth]{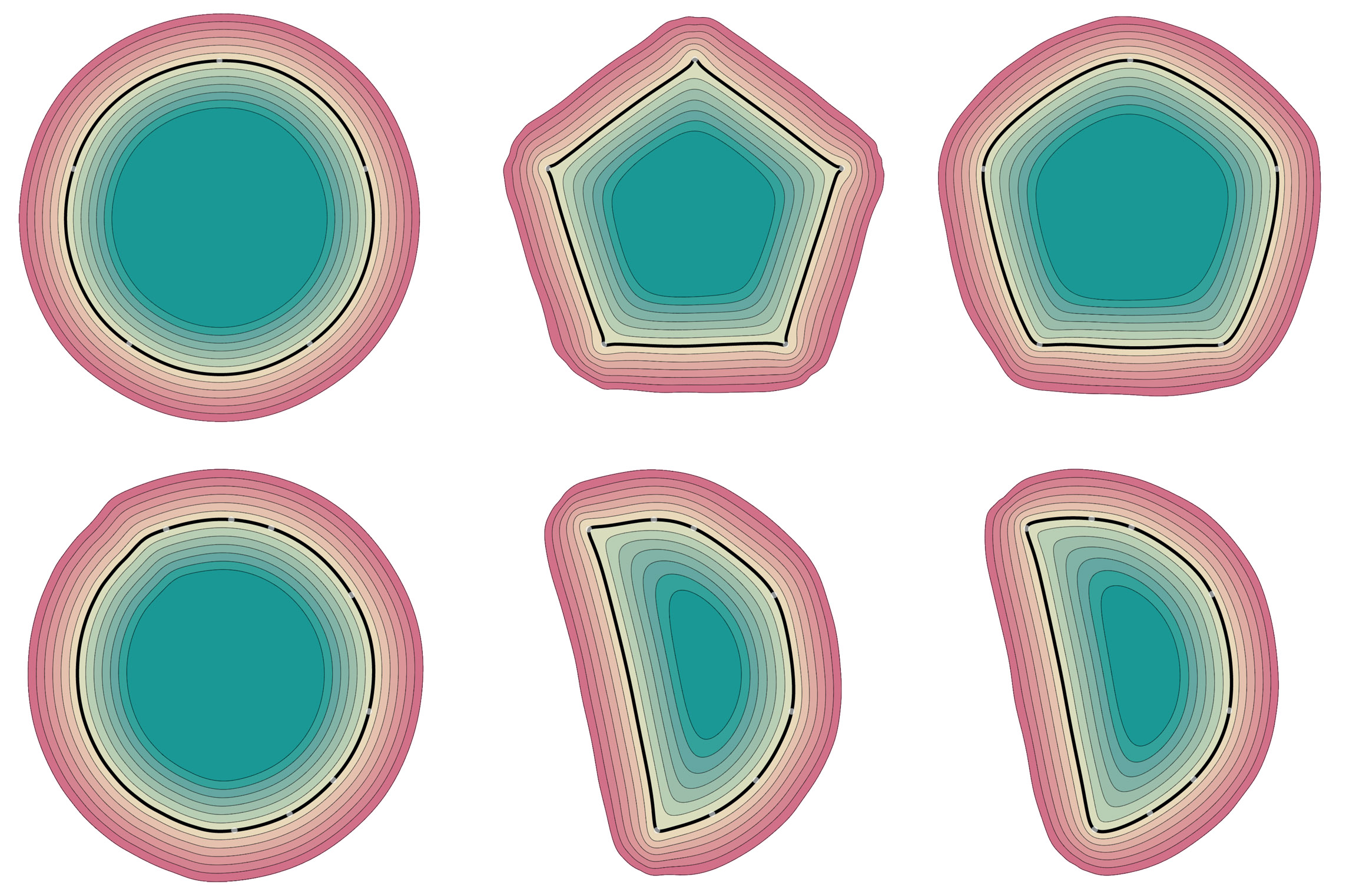}\vspace{-15pt}
    \caption{Reconstruction of 5-gon point cloud (top) and samples from half a circle (bottom). The left column shows the results of IGR, the middle columns shows PHASE without gradient loss, and right column is PHASE with gradient loss. Note the minimal perimeter property of the PHASE solutions. \vspace{-15pt}}
    \label{fig:ngon}
\end{figure}

\section{Experiments}

We present experiments for $d\in\set{2,3}$ with the phase transition loss (PHASE) and compare to relevant baselines: IGR \cite{gropp2020implicit}, DGP \cite{williams2019deep}, SIREN \cite{sitzmann2020implicit}, FFN \cite{tancik2020fourfeat}, NSP \cite{williams2020neural}. For evaluation we use the same metrics as in \cite{williams2020neural}: one-sided and double-sided Chamfer ($d^\rightarrow_C(\gY_1,\gY_2)$, $d_C(\gY_1,\gY_2)$) and Hausdorff distances ($d^\rightarrow_H(\gY_1,\gY_2)$, $d_H(\gY_1,\gY_2)$), see the supplementary for exact definitions.

\begin{table}[t]\scriptsize	
    \centering
    \begin{tabular}{c|c|c|c|c|c} 
               \multicolumn{2}{c}{} & \multicolumn{2}{|c|}{Ground Truth}  & \multicolumn{2}{c}{Scans}\\ \hline
         Model & Method & $d_C$ & $d_H$ & $d_C^\too$ & $d_H^\too$  \\  \hline
         \multirow{6}{*}{Anchor}  
          & DGP & 0.33 & 8.82 & 0.08 & 2.79 \\
          & IGR & 0.22 & 4.71 & 0.12 & 1.32 \\
          & SIREN & 0.27 & 6.18 & 0.13 & 1.88 \\
          & FFN & 0.31 & 4.49 & 0.10 & 0.10 \\
          & NSP & 0.22 & 4.65 & 0.11 & 1.11 \\
          & PHASE & \textbf{0.21} & \textbf{4.29} & 0.09 & 1.23 \\ \hline
         \multirow{6}{*}{Daratech}  
          & DGP & 0.20 & 3.14 &0.04 &1.89 \\
          & IGR & 0.25 & 4.01 & 0.08 & 1.59 \\
          & SIREN & 0.29 & 4.46 & 0.12 & 1.65 \\
          & FFN & 0.34 & 5.97 & 0.10 & 0.10 \\
          & NSP & 0.21 & 4.35 & 0.08 & 1.14 \\
          & PHASE & \textbf{0.18} & \textbf{2.92} & 0.08 & 1.80 \\
          \hline
         \multirow{6}{*}{DC}  
          & DGP & 0.18 & 4.31 & 0.04 & 2.53 \\
          & IGR & 0.17 & 2.22 & 0.09 & 2.61 \\
          & SIREN & 0.18 & 2.27 & 0.09 & 1.92 \\
          & FFN & 0.20 & 2.87 & 0.10 & 0.12 \\
          & NSP & \textbf{0.14} & \textbf{1.35} & 0.06 & 2.75 \\
          & PHASE & 0.15 & 2.52 & 0.05 & 2.78 \\
          \hline
         \multirow{6}{*}{Gargoyle}  
          & DGP & 0.21 & 5.98 & 0.06 & 3.41 \\
          & IGR & \textbf{0.16} & 3.52 & 0.06 & 0.81 \\
          & SIREN & 0.29  & 3.90 & 0.13 &1.93 \\
          & FFN & 0.22 &5.04& 0.09 &0.09 \\
          & NSP & \textbf{0.16} &3.20& 0.08& 2.75 \\
          & PHASE & \textbf{0.16} & \textbf{3.14} & 0.07 & 1.09 \\
          \hline
         \multirow{6}{*}{Lord Quas}  
          & DGP & 0.14 & 3.67 & 0.04 & 2.03 \\
          & IGR & 0.12 & 1.17 & 0.07 & 0.98 \\
          & SIREN & 0.13 & 0.89 & 0.06 & 0.96 \\
          & FFN & 0.35 & 3.90 & 0.06 & 0.06 \\
          & NSP & 0.12 & \textbf{0.69} & 0.05 & 0.62 \\
          & PHASE & \textbf{0.11} & 0.96 & 0.04 & 0.96 
    \end{tabular}\vspace{-5pt}
    \caption{Surface reconstruction results on the benchmark of \cite{williams2019deep}. }\vspace{-10pt}
    \label{tab:dgp}
\end{table}

\subsection{2D evaluation}

Figure \ref{fig:u_ns_w} depicts the PHASE learned densities $u_\eps$ (bottom row) and their log transforms $w_\eps$ defined in \eqref{e:w} (top row). The input point cloud in each example is depicted with white dots in the bottom row. In this example we use the loss in \eqref{e:loss} with $\mu=0.1$, and $\lambda=0.3$. Note that indeed $u_\eps$ approximates an occupancy function, while $w_\eps$ approximates a signed distance function.

Figure \ref{fig:ngon} compares the reconstructions of IGR (left), and PHASE with $\mu=0$ (middle), and PHASE with $\mu=0.1$ (right), the input point cloud in each example is depicted with white dots; the zero level-set is in bold; in both cases $\lambda=0.3$. Note that the PHASE results indeed exhibit smaller perimeter surface, while the IGR results tend to produce higher area extrapolation. Further note that incorporating the gradient loss $\gN_{\text{unit}}$ (with no normal data) increase the regularity of the reconstructed $\gS$ at the data points. 

Figure \ref{fig:grads} depicts the deviation of the gradient norm, $\norm{\nabla w_\eps}$, from $1$ (right column) when $\mu=0$. As anticipated by Theorem \ref{thm:visc_eikonal} (by excluding the union of balls $B_\vx$) the gradient norm deviates from $1$ more strongly closer to the data points. As described above, this provides some justification to the losses in equations \ref{e:N_intr}-\ref{e:N_unit}. Another area of higher error is the medial axis, namely the gradient discontinuity locus of the signed distance function. This can be explained by the fact that $w_\eps$, is a solution to the viscosity Eikonal PDE (\eqref{e:w_eps_viscosity_eikonal}) and therefore produces smooth solutions, in contrast to the SDF in this region. Note, that there is no smooth function satisfying $\norm{\nabla f}=1$ everywhere.

\begin{figure}
    \centering
    \begin{tabular}{ccc}
         \includegraphics[width=0.3\columnwidth]{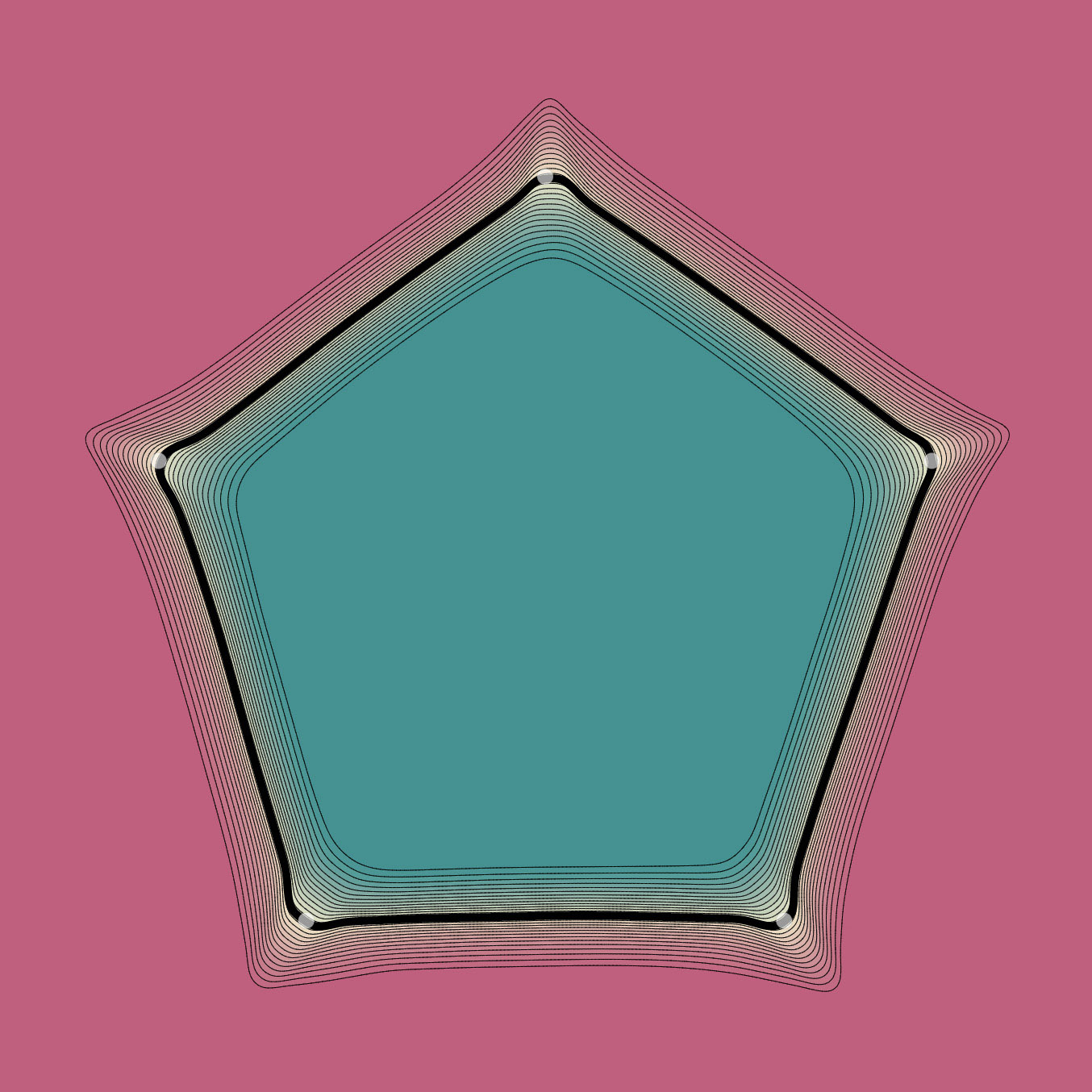} &
         \includegraphics[width=0.3\columnwidth]{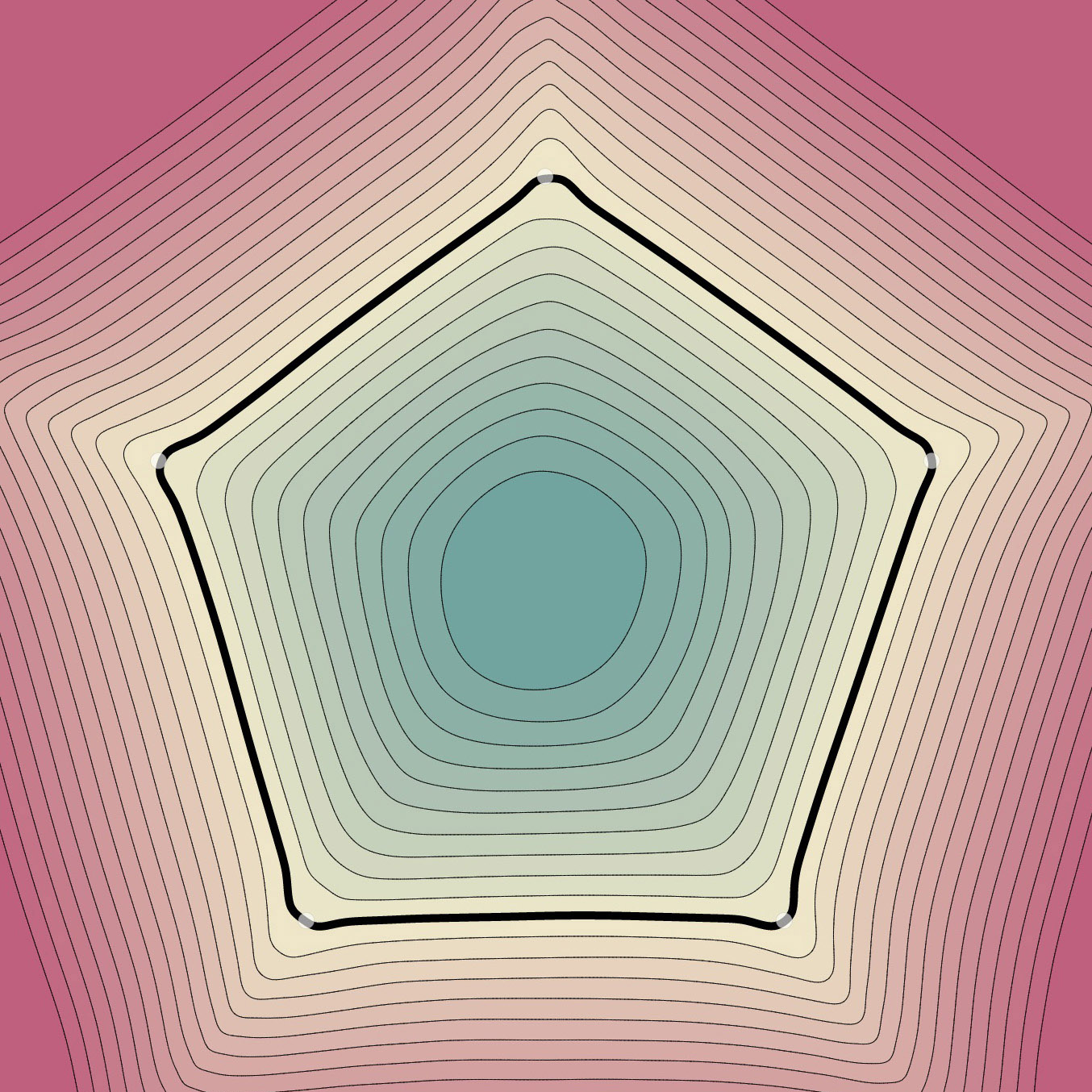} &
         \includegraphics[width=0.3\columnwidth]{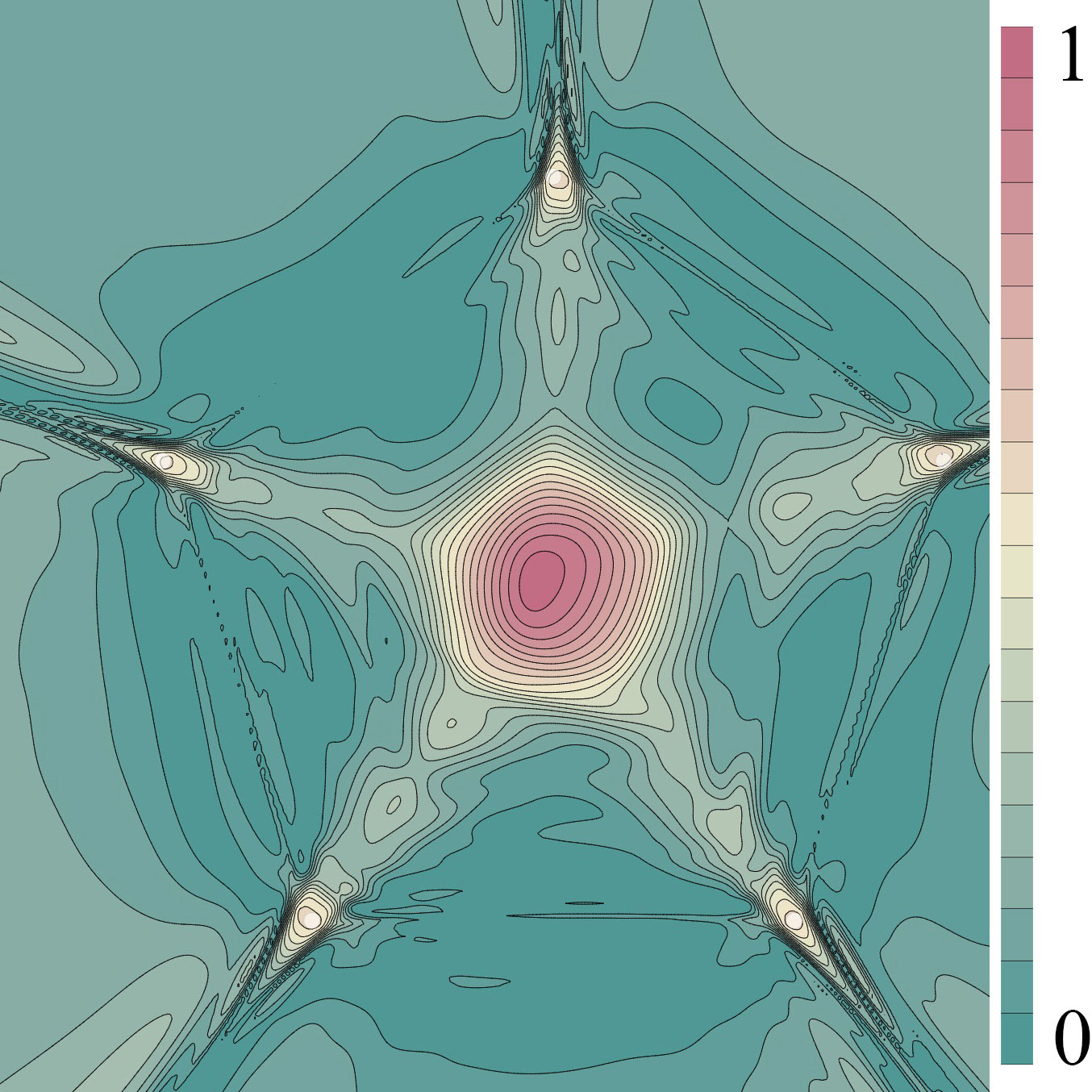} \\
         \includegraphics[width=0.3\columnwidth]{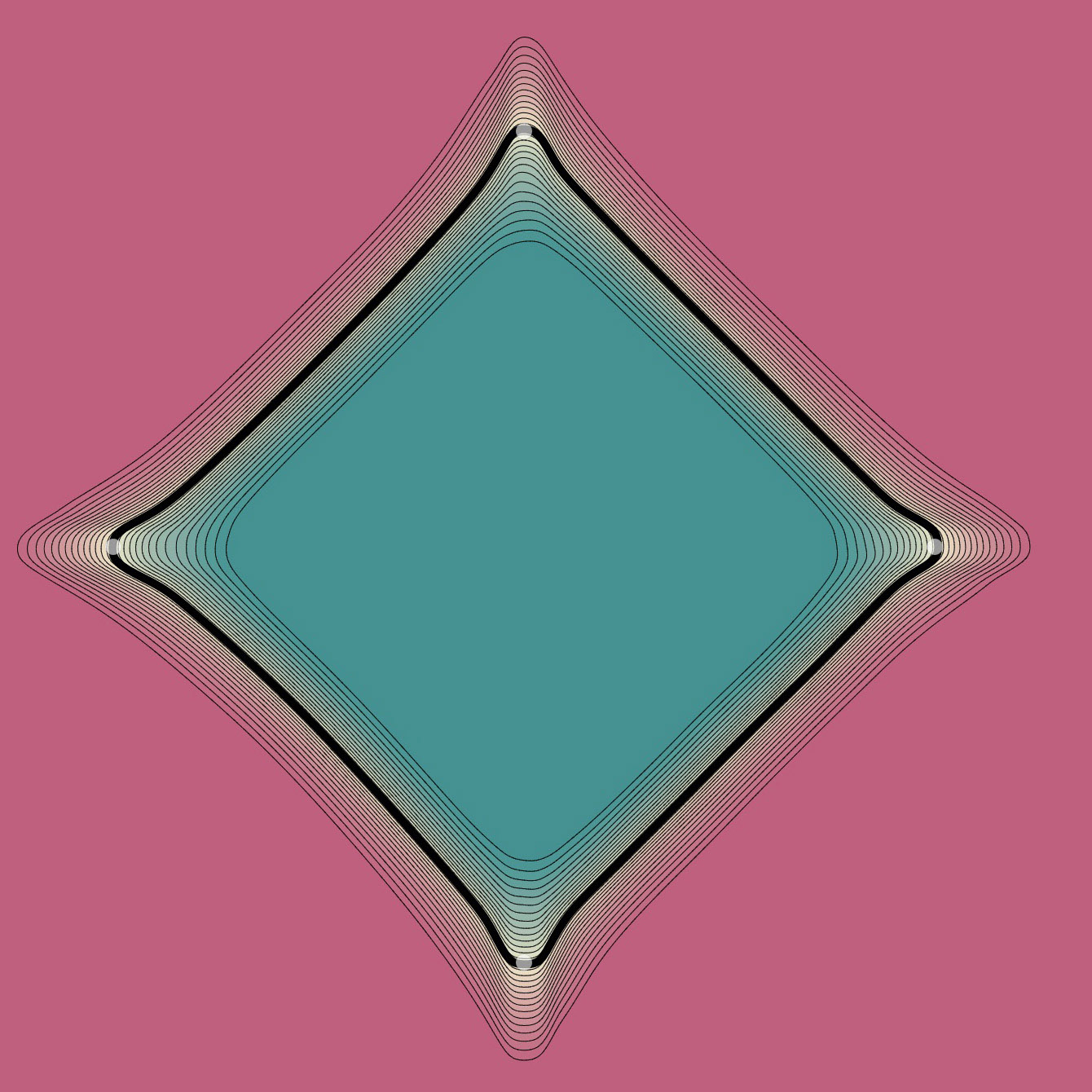} &
         \includegraphics[width=0.3\columnwidth]{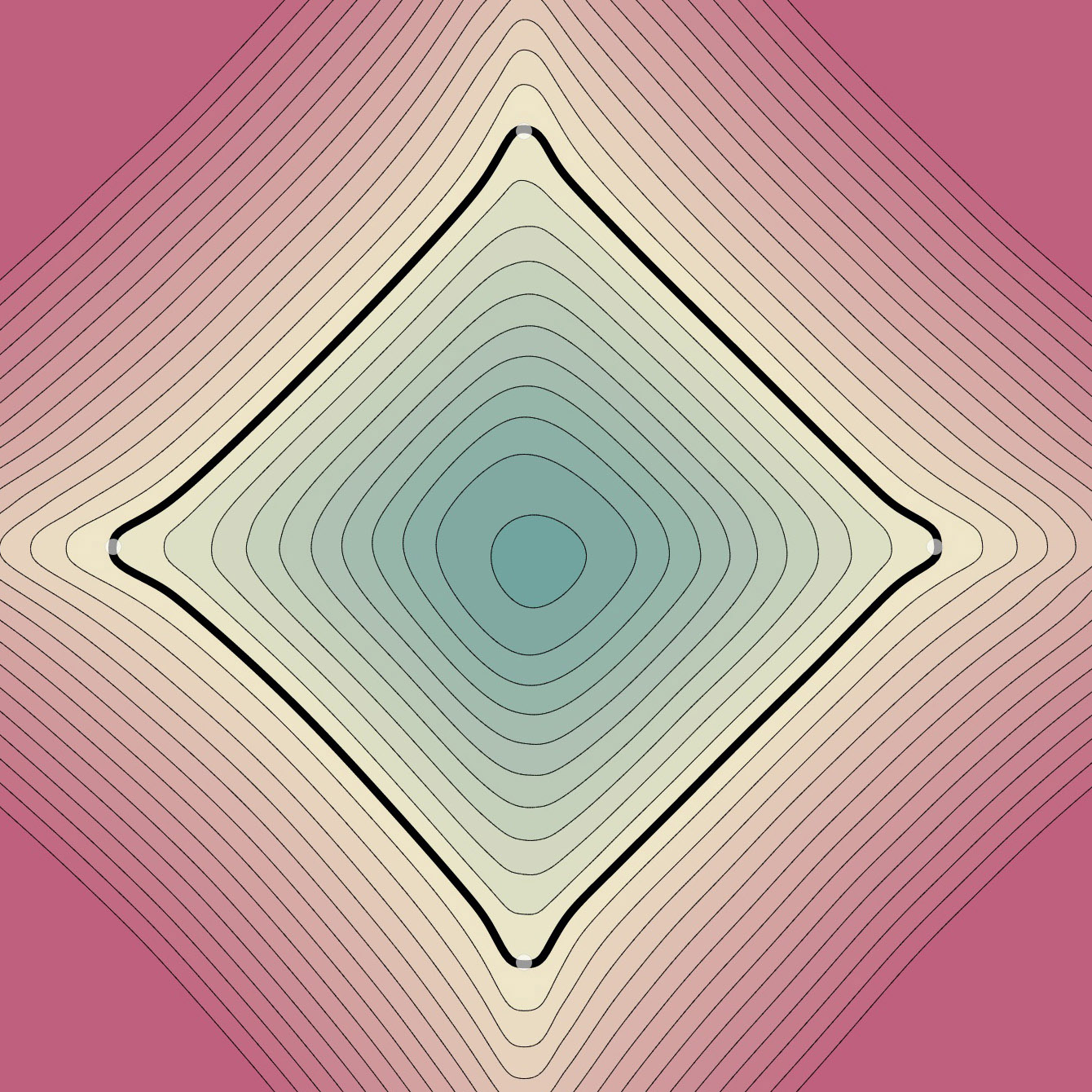} &
         \includegraphics[width=0.3\columnwidth]{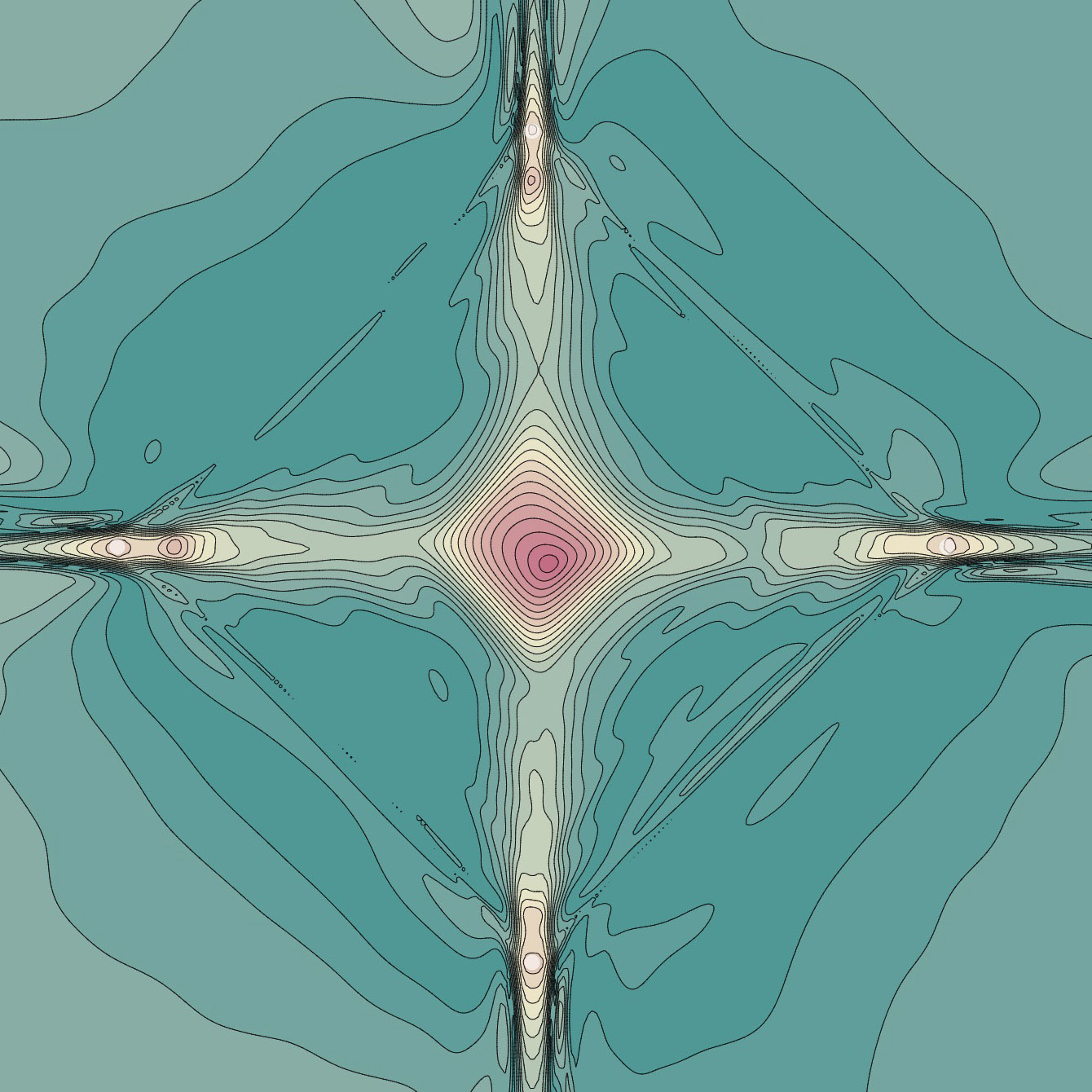} \\
         \includegraphics[width=0.3\columnwidth]{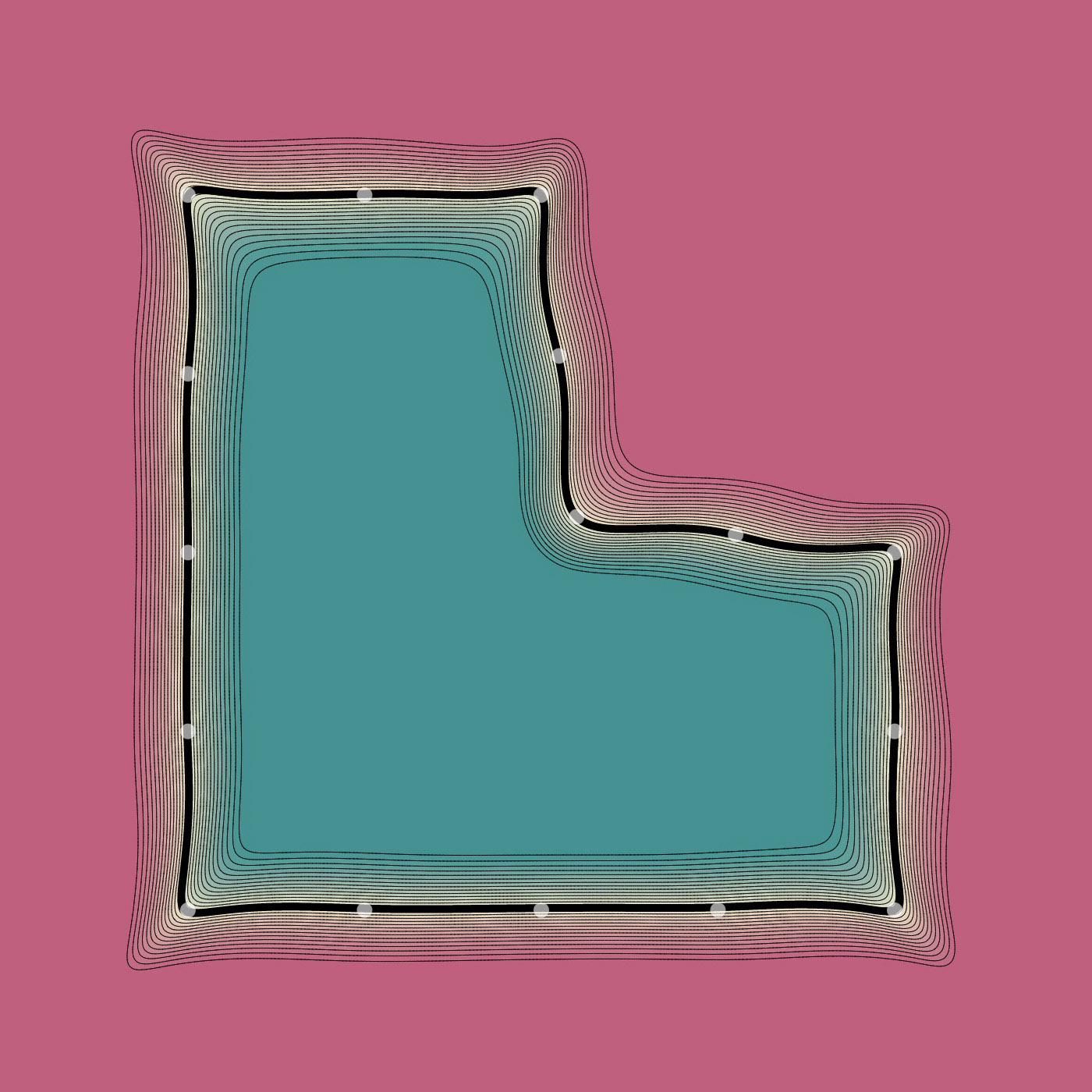} &
         \includegraphics[width=0.3\columnwidth]{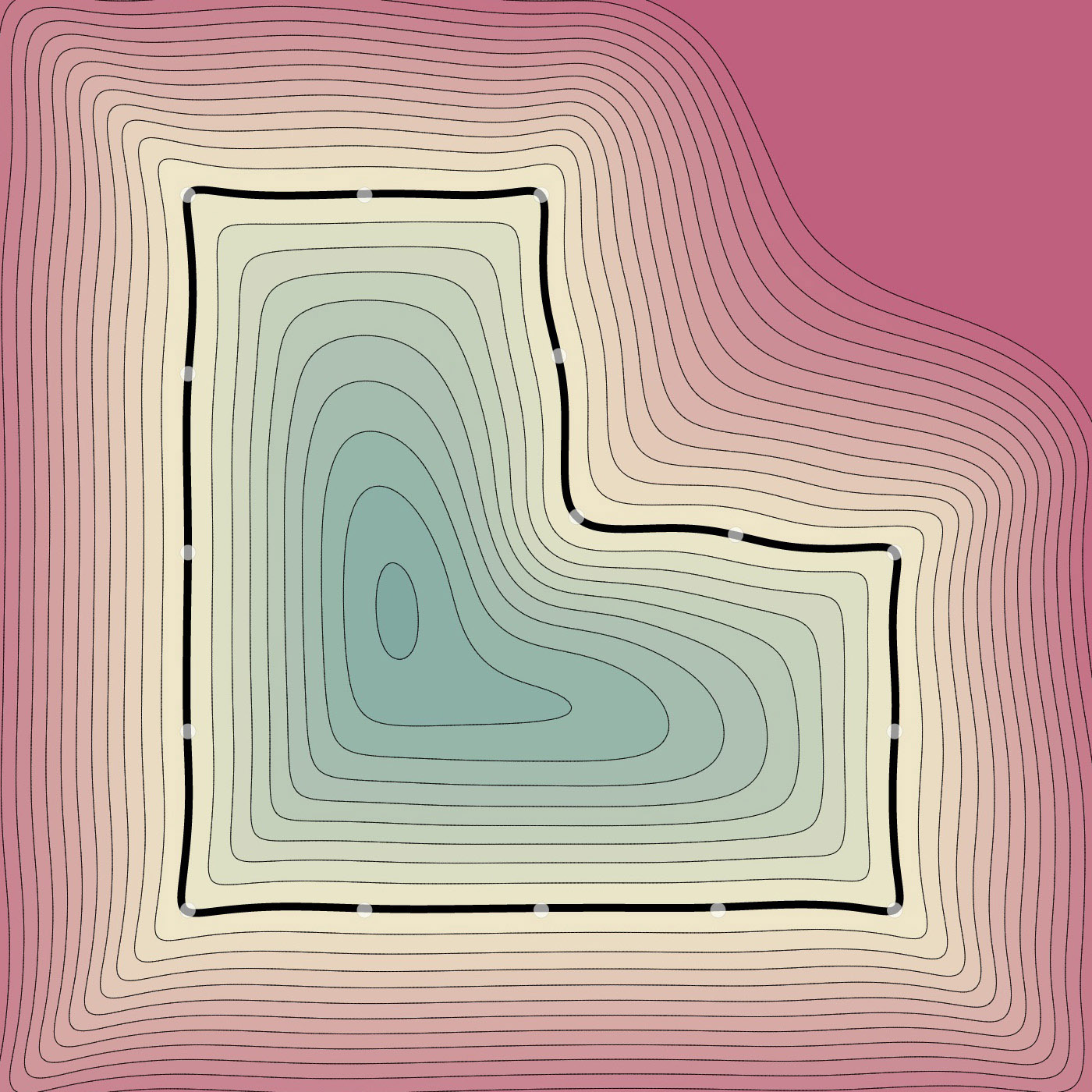} &
         \includegraphics[width=0.3\columnwidth]{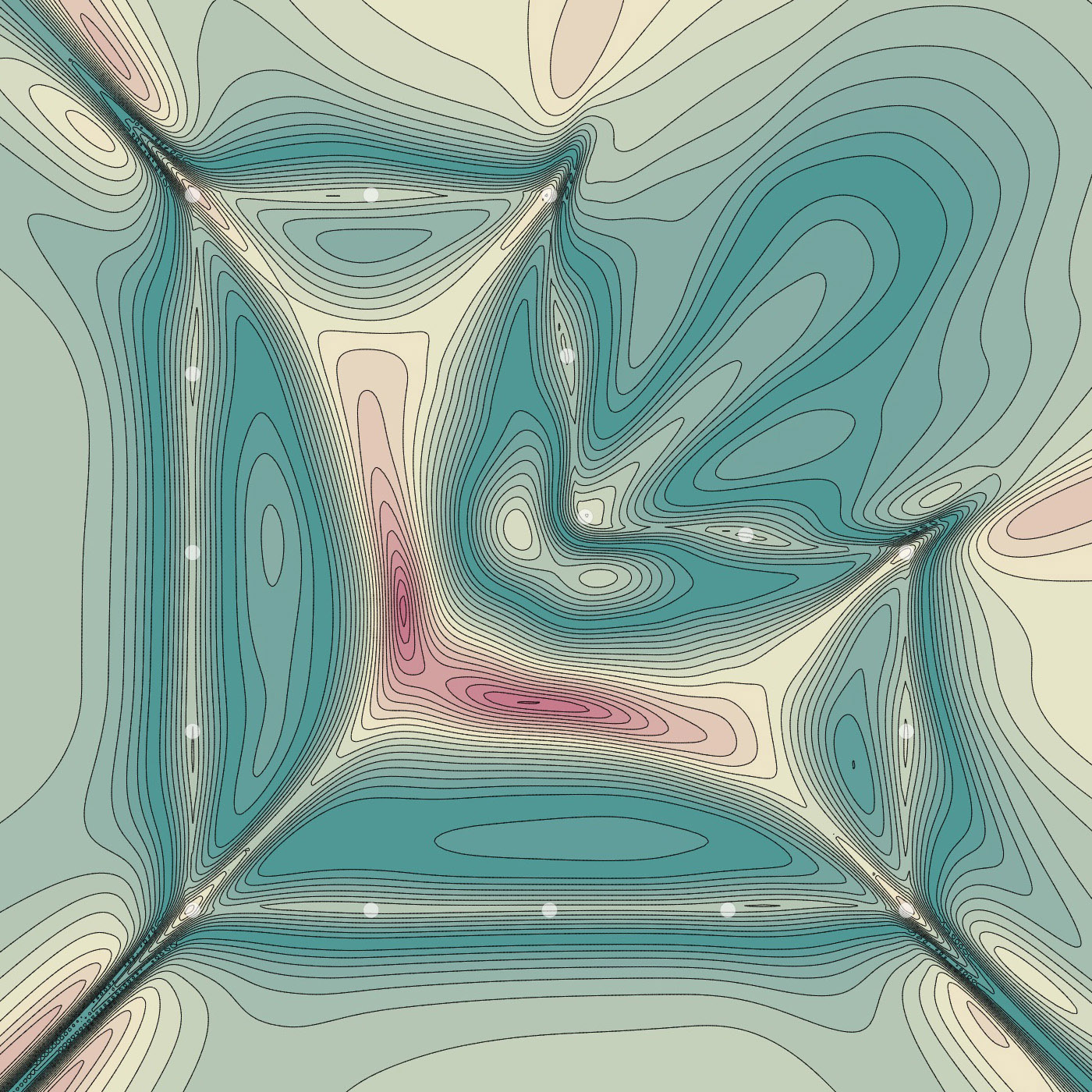} \\
         $u_\eps$ & $w_\eps$ & $\abs{\norm{\nabla w_\eps}-1}$
    \end{tabular}
    \caption{We show the deviation of the gradient norm of $w_\eps$ from $1$ (right), for three different 2D point clouds. Note the deviation from unit gradient near the sample points and medial axis. }
    \label{fig:grads}
\end{figure}

\subsection{Surface reconstruction benchmark} 
\label{ss:surface_reconstruction_benchmark}
We evaluated our loss on the dataset in \cite{williams2019deep}, and compared to relevant baselines. This dataset consists of noisy range scans, each including a point cloud $\gX$ with corresponding normals $\vn$. The data set contains surfaces with complex geometry and topology. Since normals are available we use in this case the loss in \eqref{e:loss} with $\gN_{\text{intr}}$ from \eqref{e:N_intr}. We note that DGP already established itself as superior to many classical surface reconstruction methods. Furthermore, the results of the baselines are taken from NSP that performed parameter sweep for all methods (except itself, which is not an INR method per se) and chose individual best parameter for each model; we used the \emph{same} set of parameters for all the models, \ie, $\lambda=\mu=10$, and trained for $100k$ iterations. Results are summarized in Table \ref{tab:dgp}. Note that we achieve best reconstruction error in $4$ out of $5$ cases for Chamfer distance, and $3$ out of $5$ in Hausdorff distance. 

\begin{figure}
    \centering
    \begin{tabular}{@{\hskip0pt}c@{\hskip0pt}c@{\hskip0pt}c@{\hskip0pt}c@{\hskip0pt}}
        \includegraphics[width=0.24\columnwidth]{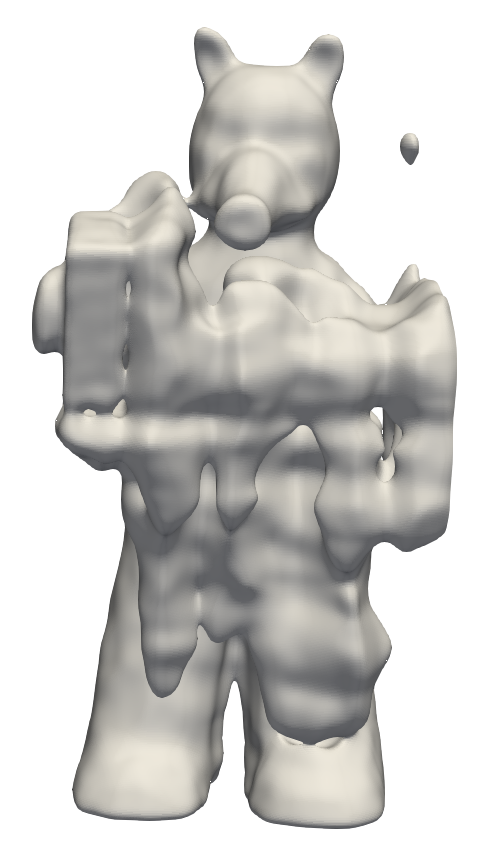} & 
        \includegraphics[width=0.24\columnwidth]{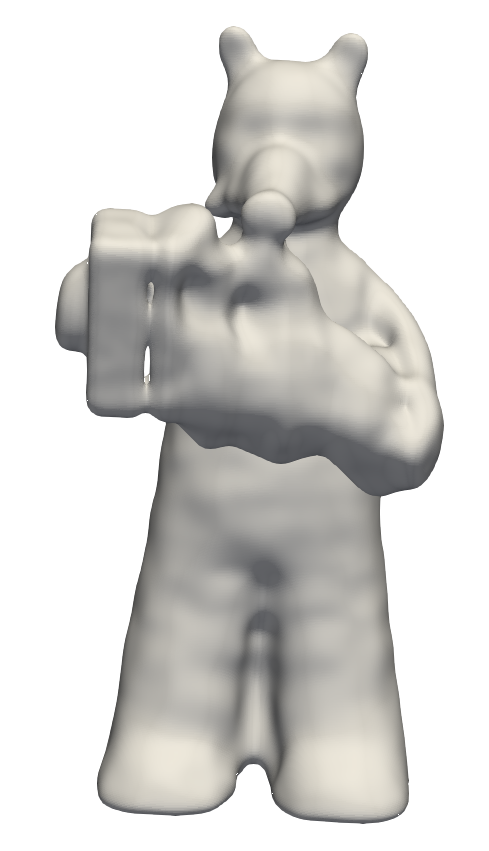} & 
        \includegraphics[width=0.24\columnwidth]{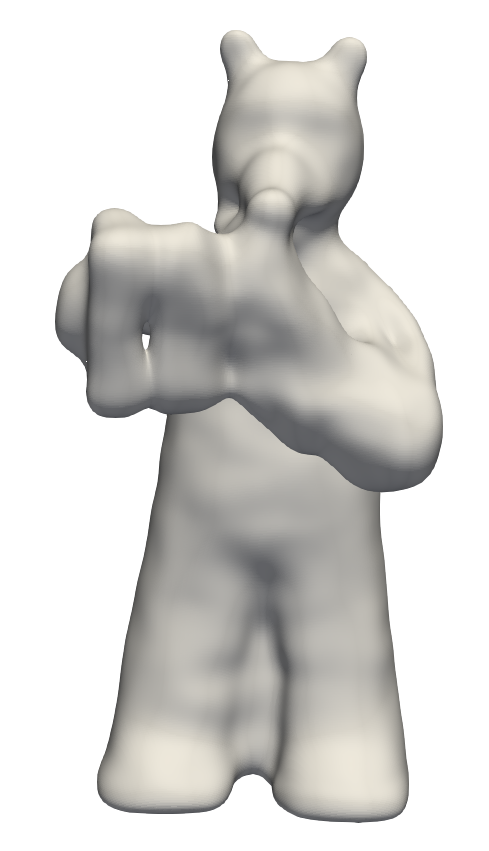} & 
        \includegraphics[width=0.24\columnwidth]{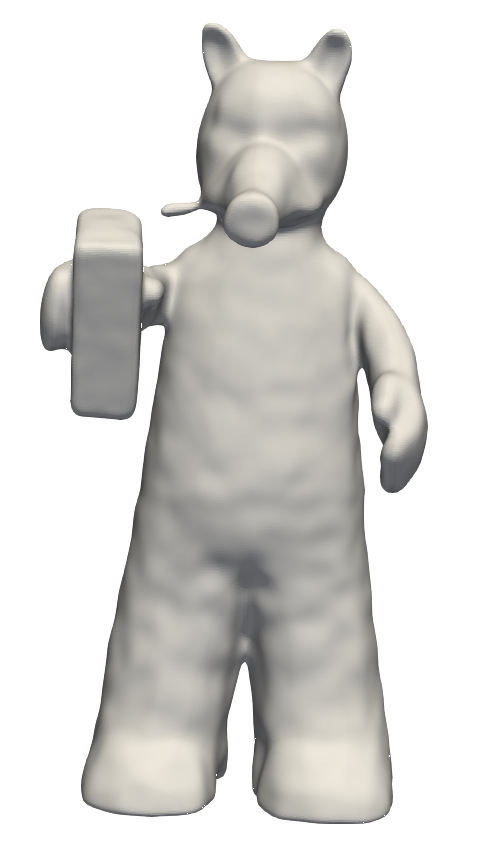} \vspace{-10pt} \\
        {\scriptsize(a)} & {\scriptsize(b)} & {\scriptsize(c)} & {\scriptsize(d)} \vspace{-8pt}
    \end{tabular}
    \caption{Reconstruction with no normals and Fourier features. (a-c) is IGR+FF, note the extraneous parts; (d) is PHASE+FF, removal of extraneous parts is attributed to the minimal perimeter. \vspace{-10pt}}
    \label{fig:ff}
\end{figure}

\subsection{Learning from points clouds and Fourier features}
In this section we work with the surface reconstruction benchmark (see Section \ref{ss:surface_reconstruction_benchmark}) but this time in the more challenging scenario of using only point cloud data (without normals). We compared PHASE with IGR together with high frequency learning that allow fast learning of details; we used Fourier features as described in Section \ref{s:implementation_details} with $k=6$, and trained for $10k$ iterations. Similarly to previous work \cite{williams2020neural}, we notice that in the point cloud only case the high frequency method possess an inductive bias that tends to introduce extraneous surface parts. For example, Figure \ref{fig:ff} (a), (b), and (c) show the results of IGR training with Fourier features (FF); (d) shows the result of training PHASE with the same Fourier feature map. We attribute the reduction in the extraneous parts to the minimal surface perimeter property of our loss. Since this is a point only case, we trained \eqref{e:loss} with \eqref{e:N_unit}; we note that in this case high $\mu$ values cause some instability, and we conducted the experiment with $\lambda=10$, and $\mu=0.5$; we present more results including also failure cases in the supplementary. Table \ref{tab:ff} provides the quantitative success of PHASE and IGR with Fourier features on this benchmark; for IGR we did a parameter sweep and chose the parameters that achieved optimal $d_C, d_H$ on the Gargoyle model. Note that PHASE is on par with some of the methods that use normal data in Section \ref{ss:surface_reconstruction_benchmark}.

\begin{table}[]\scriptsize	
    \centering
    \begin{tabular}{c|c|c|c|c|c} 
               \multicolumn{2}{c}{} & \multicolumn{2}{|c|}{Ground Truth}  & 
               \multicolumn{2}{c}{Scans}\\ \hline 
         Model & Method & $d_C$ & $d_H$ & $d_C^\too$ & $d_H^\too$  \\  \hline
         \multirow{2}{*}{Anchor}  
          & IGR+FF & 0.72 & 9.48 & 0.24 & 8.89 \\
          & PHASE+FF & 0.29 & 7.43 & 0.09 & 1.49 \\ \hline
         \multirow{2}{*}{Daratech}  
           & IGR+FF & 2.48 & 19.6 & 0.74 & 4.23 \\
          & PHASE+FF & 0.35 & 7.24 & 0.08 & 1.21 \\ \hline
         \multirow{2}{*}{DC}  
         & IGR+FF & 0.86 & 10.32 & 0.28 & 3.98 \\
          & PHASE+FF & 0.19 & 4.65 & 0.05 & 2.78 \\ \hline
         \multirow{2}{*}{Gargoyle}  
         & IGR+FF & 0.26 & 5.24 & 0.18 & 2.93 \\
          & PHASE+FF & 0.17 & 4.79 & 0.07 & 1.58 \\ \hline
         \multirow{2}{*}{Lord Quas}  
         & IGR+FF & 0.49 & 10.71 & 0.14 & 3.71 \\
          & PHASE+FF & 0.11 & 0.71 & 0.05 & 0.74 \\ \hline
    \end{tabular}
    \caption{Using Fourier features for fast training on the benchmark of \cite{williams2019deep} using only point data (no normals).\vspace{-8pt} }
    \label{tab:ff}
\end{table}

\subsection{Large point clouds}
In this experiment we trained on point clouds and normals extracted from large models taken from the Stanford 3D Scanning Repository (Source: Stanford University Computer Graphics Laboratory). The point clouds consist of $0.5m$-$14m$ vertices, and were trained for $100k$ iteration of $1k$ batches with Fourier features $k=6$. Figure \ref{fig:stanford} shows qualitative result of these reconstructions. Notably, the largest model, "Lucy" (left, $14m$ points) was trained for only $7$ epochs and still presents relatively high level of details and no visible artifacts. We used the same parameters as in Section \ref{ss:surface_reconstruction_benchmark}, \ie, $\lambda=\mu=10$.  

In Figure \ref{fig:room} we show the PHASE reconstruction of a room scene point cloud from \cite{sitzmann2020implicit} with roughly $10m$ points. In this experiment we have followed their set-up and used an MLP with $5$ layers of $1024$ neurons each. We have trained the loss with $\lambda=10$, $\mu=1$, for $100k$ iterations and batch size of $15k$, and Fourier features $k=6$.  

\begin{figure}
    \centering
    \includegraphics[width=\columnwidth]{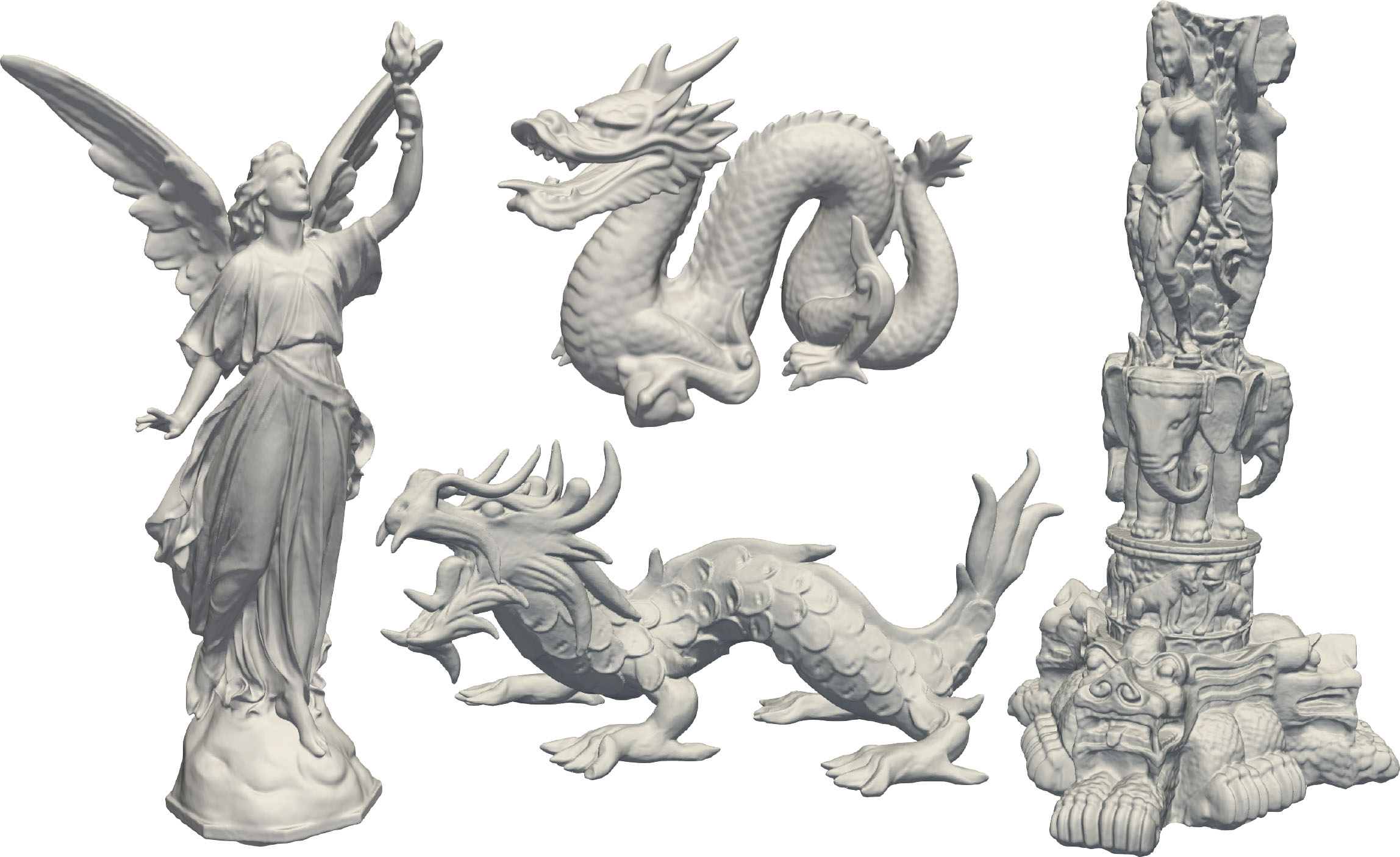}\vspace{-15pt}
    \caption{Reconstruction of large models from the Stanford 3D Scanning Repository. }
    \label{fig:stanford}
\end{figure}

\begin{figure}[t]
    \centering
    \includegraphics[width=1.0\columnwidth]{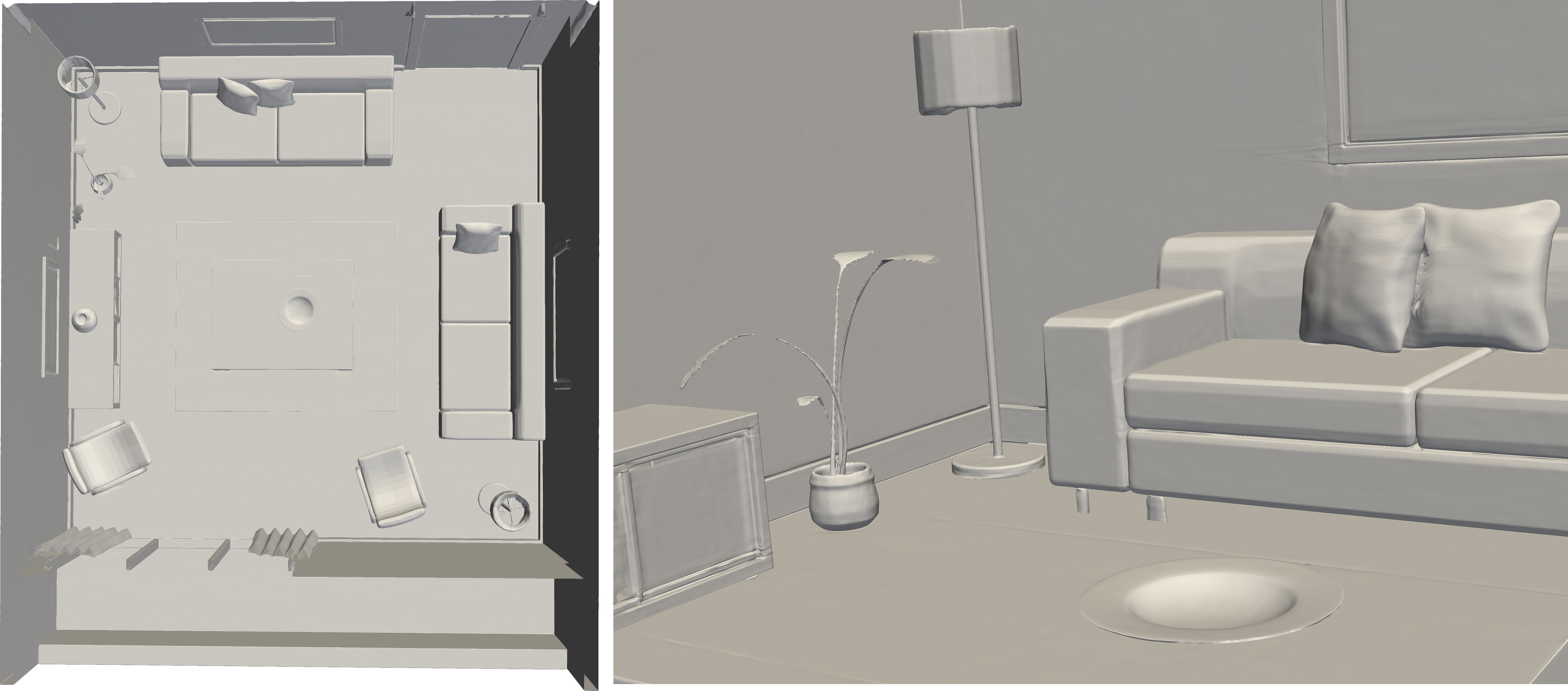} 
    \caption{PHASE reconstruction of the room scene from SIREN \cite{sitzmann2020implicit}.}
    \label{fig:room}
\end{figure}

\section{Conclusions}
We presented a new loss for surface reconstruction from raw data, suitable for training implicit neural representations. This new loss learns a signed density function and is motivated from a mathematical physics model of fluids equilibrium in a container. This theory provides a principled tool to analyze the limit behavior of the loss minimizers. Most notably, convergence to occupancy function with minimal surface perimeter. By choosing a specific energy potential $W$  we show that the log transform of the learned density approximates the signed distance function. Experiments with this loss demonstrate its usefulness for learning INRs from raw data. 
Interesting future work is to expand the analysis of the PHASE loss to include other reconstruction terms (\eg, equations \ref{e:N_intr} and \ref{e:N_unit}); another interesting venue is to connect this loss to prior losses, possibly better understanding their limit behaviour and bias. A more ambitious goal would be to develop a general mathematical theory for INR losses, with well understood bias that can be tuned and even learned from data. 

\bibliography{bibliography}
\bibliographystyle{icml2021}

\cleardoublepage

\section{$\Gamma$-convergence} 
\begin{proof}[Proof (Theorem \ref{thm:main_gamma})]
We will adapt the proof of Theorem 13.6 in \cite{rindler2018calculus} to the surface reconstruction loss. We want to prove $\Gamma$-convergence of $\eps^{-1/2}\gF_\eps$ to $\gF_0$. Note that $\eps^{-1/2}\gF_\eps$ is equivalent to the functional 
\begin{equation}\label{e:F_eps_PROOF}
    \gF_{\epsilon}(u)=\lambda\gL(u)+\begin{cases} \int_\Omega \eps\norm{\nabla u}^2 + \frac{1}{\eps}W(u) & u\in W^{1,2}(\Omega)\\ 
    +\infty & \text{otherwise}
    \end{cases}
\end{equation}
where now $\lambda(\eps) = \eps^{-1}\tilde{\lambda}(\eps^2)$,  where $\tilde{\lambda}$ represents the dependence in \eqref{e:F_eps} (in the main paper). So we will prove $\Gamma$-convegence of this $\gF_\eps$ to $\gF_0$, where $\lambda\too\infty$ and $\lambda\sqrt{\eps}\too 0$, as $\eps \dtoo 0$. The proof of $\Gamma$-convergence requires showing the \emph{$\liminf$} and \emph{recovery} properties from Section \ref{ss:gamma_convergence} in the main paper. 
Let us denote:
\begin{align*}
  \gE_\eps(u) &= \begin{cases} \int_\Omega \eps\norm{\nabla u}^2 + \frac{1}{\eps}W(u) & u\in W^{1,2}(\Omega)\\ 
    +\infty & \text{otherwise}
    \end{cases} \\
\gE_0(u) &= \begin{cases} \sigma_0 \per_\Omega(\gI) & u\in \BV(\Omega;\set{-1,1}) \\
    +\infty & \text{otherwise}
    \end{cases}
\end{align*}

For the reader's convenience we also repeat the definition of the limit functions $\gF_0$:
\begin{equation}
    \gF_0(u)=  \begin{cases} \sigma_0 \per_\Omega(\gI) & u\in \BV(\Omega,\set{-1,1}), \\&  \text{and } \gL(u)=0 \\ 
    +\infty & \text{otherwise}
    \end{cases}
\end{equation}

\textbf{\textbf{\textrm{Lim inf}} part.}
In this part we need to consider $u_\eps \too u$ in $L^1(\Omega)$; we abuse notation a bit and let $\eps\dtoo 0$ denote some particular sequence $\eps_k\too 0$ as $k\too \infty$. We need to show that $\liminf_{\eps\dtoo 0} \gF_\eps(u_\eps) \geq \gF_0(u)$. If $\liminf_{\eps\dtoo 0} \gF_\eps(u_\eps)=\infty$ the statement holds trivially, therefore we assume $\liminf_{\eps\dtoo 0} \gF_\eps(u_\eps)<\infty$. 

In the proof of Theorem 13.6 in \citet{rindler2018calculus} it is shown that $\int_\Omega W(u(\vx))d\vx = 0$ and consequently $u(\vx)\in\set{-1,+1}$ almost everywhere, \ie, \eqref{e:u_0}  (in the main paper) holds for $u$. Furthermore, it is shown that  
\begin{equation}\label{e:liminf_E}
  \liminf_{\eps\dtoo 0} \gE_\eps(u_\eps) \geq \gE_0(u),   
\end{equation}
where $\sigma_0$ is some constant depending on $W$ alone:
\begin{equation}
\sigma_0 = 2\int_{-1}^1 \sqrt{W(s)}ds.
\end{equation}
Since $\gE_0(u)<\infty$, $u\in \BV(\Omega;\set{-1,1})$ (this can be seen directly from the definitions in equations \ref{e:def_BV} and \ref{e:per} in the main paper).

Now, for our reconstruction loss, since $u_\eps\too u$ in $L^1(\Omega)$ we have that for every $\vx\in\gX$, $\abs{\int_{B_\vx}u_\eps} \too \abs{\int_{B_\vx}u}$, as $\eps\dtoo 0$. By Fatou's lemma applied for the functions $\vx\mapsto \abs{\int_{B_\vx}u_\eps}$ and the limit function $\vx\mapsto \abs{\int_{B_\vx}u}$, and the fact that $\liminf_{\eps\dtoo 0}\gF_\eps(u_\eps)<\infty$:
\begin{align*}
  \gL(u) &= \E_\vx \abs{\int_{B_\vx}u} \leq \liminf_{\eps\dtoo 0} \E_\vx \abs{\int_{B_\vx}u_\eps}\\ & = 
  \liminf_{\eps\dtoo 0} \gL(u_\eps) \leq \liminf_{\eps\dtoo 0}\frac{1}{\lambda}\gF_\eps(u_\eps)  =0.
\end{align*}
where the last equality is due to $\lambda\too \infty$ as $\eps\dtoo 0$. This means that the limit function $u$ satisfies the reconstruction constraints perfectly, or in other words that the reconstructed surface $\gS$ passes through all the balls $B_\vx$, $\vx\in \gX$, except possibly a subset of $\gX$ of measure zero. However, since $\abs{\int_{B_\vx}u}$ is continuous as a function of $\vx$ this is true for all balls $B_\vx$, $\vx\in\gX$. 

In particular $\gE_0(u) = \gF_0(u)$. Now incorporating this with \eqref{e:liminf_E} we get 
\begin{align*}
  \liminf_{\eps\dtoo 0} \gF_\eps(u_\eps) & = \liminf_{\eps\dtoo 0}  \parr{\lambda\gL(u_\eps) + \gE_\eps(u_\eps)} \\
  &\geq \liminf_{\eps\dtoo 0}  \gE_\eps(u_\eps) \\ &\geq \gE_0(u)  = \gF_0(u)  
\end{align*}
as required.

\textbf{Recovery sequence part.}
In this part we need to consider an arbitrary $u\in L^1(\Omega)$ and find a sequence $u_\eps\in L^1(\Omega)$ so that $u_\eps\too u$ in $L^1(\Omega)$ and $\lim_{\eps\dtoo 0}\gF_\eps(u_\eps)=\gF_0(u)$.

Let $u\in L^1(\Omega)$ be arbitrary. If $u\notin \BV(\Omega;\set{-1,1})$ or $\gL(u)>0$ then $\gF_0(u)=\infty$ and there is no need to construct a recovery sequence in this case (see, \eg, Theorem \ref{thm:gamma_conv_of_mins} where no recovery sequence for such cases is needed). So we assume $\gF_0(u)<\infty$, and $u$ of the form $u = -\one_\gI + \one_{\Omega\setminus \gI}$, where $\gI$ is defined as in \eqref{e:I_and_O}.

Next, note that if $\int_\Omega u \in \set{-|\Omega|,|\Omega|}$ then $\gL(u)>0$ and again $\gF_0(u)=\infty$. Therefore we can assume $\int_\Omega u \in (-|\Omega|,|\Omega|)$. In this case, Theorem 13.6 in \citet{rindler2018calculus} shows the existence of a recovery sequence of functions $u_\eps$ so that $u_\eps\too u$ in $L^1(\Omega)$ and $\lim_{\eps\dtoo 0}\gE_\eps(u_\eps)=\gE_0(u)$. The main observation in this part is that $u_\eps$ is a recovery sequence also for our surface reconstruction functionals $\gF_\eps$ and $\gF_0$. 

To show that $u_\eps$ is a recovery sequence also in our settings it is enough to show that $\lambda\gL(u_\eps)\too 0$ as $\eps\dtoo 0$. Indeed, if this is the case, 
\begin{align*}
  \lim_{\eps \dtoo 0}\gF_\eps(u_\eps) &= \lim_{\eps \dtoo 0}\parr{\gE_\eps(u_\eps) + \lambda\gL(u_\eps)} \\ &=\gE_0(u)=\gF_0(u)  
\end{align*}

To show that $\lambda\gL(u_\eps)\too 0$ we need to use a bit of extra information on $u_\eps$: $u_\eps$ is constructed to approximate $w=-\one_{G\cap \Omega}+\one_{\Omega\setminus G}$, that is 
$$\int_\Omega |u_\eps - w| \too 0,$$ 
where $G\subset \Real^d$ is open, bounded with smooth boundary, and $\abs{(G\,\Delta\,\gI)\cap \Omega}$ can be made arbitrary small (see Lemma 13.7 in \citet{rindler2018calculus}, or Lemma 1 in \citet{modica1987gradient}). This means that for arbitrary $\delta>0$ we can choose $G$ so that 
$$\int_\Omega \abs{u-w} \leq \abs{(G\,\Delta\,\gI)\cap \Omega} \leq \delta$$
Remember that $\gL(u)=0$ and using this last inequality we get that for arbitrary $\vx\in\gX$, 
\begin{equation}\label{e:int_B}
    \begin{aligned}
    \abs{\int_{B_\vx} u_\eps} &=  \abs{\int_{B_\vx} u_\eps - \int_{B_\vx} w} + \abs{\int_{B_\vx} w - \int_{B_\vx} u} \\ &\leq \int_\Omega \abs{u_\eps-w} + \int_\Omega \abs{w-u} \leq c\sqrt{\eps} 
    \end{aligned}
\end{equation}
where $c>0$ is some constant, and the last inequality is due to the fact that we can make the choice $\delta=\sqrt{\eps}$ and the following bound shown in the proof of Theorem 13.6 in \citet{rindler2018calculus}:
$$\int_\Omega \abs{u_\eps-w}\leq 4\sqrt{\eps}\sup_{-2\sqrt{\eps}\leq t\leq 2\sqrt{\eps}}\gH^{d-1}(\gS_t\cap \Omega),$$
where $\gS_t=\set{\vx\in\Real^d \ \vert \ d_\gS(\vx)=t}$,  $d_\gS$ is the signed distance function defined in \eqref{e:sdf}, $\gS$ is defined as in \eqref{e:I_and_O} for $w$ , and Lemma 13.9 in \citet{rindler2018calculus} shows that the Hausdorff measure of $\gS_t\cap\Omega$ converges to that of $\partial\gI\cap\Omega$ as $\eps\dtoo 0$, and therefore is bounded. 

\Eqref{e:int_B} implies that
$$\lambda\gL(u_\eps)\leq c \lambda \sqrt{\eps}.$$
Lastly, remember that $\lambda \sqrt{\eps} \too 0$ and therefore $u_\eps$ is a recovery sequence as desired. 

\end{proof}


\section{Distance functions}

\begin{reptheorem}{thm:laplace_u}
Let $u_\eps\in W^{1,2}(\Omega)$ be a (local) minimizer of $\gF_\eps$, and $O\subset\Omega \setminus \cup_{\vx\in\gX} B_\vx$ a set where $u_\eps\ne 0$. Then, $u_\eps$ is smooth in the classical sense in $O$ and satisfies 
\begin{equation}\label{e:linear_}
-\eps\Delta u_\eps + u_\eps -\sign(u_\eps) = 0.
\end{equation}
\end{reptheorem}
\begin{proof}
We start by applying the Euler-Largrange (EL) conditions (see \eg, Theorem 3.1 in \cite{rindler2018calculus}) for a minimizer $u$ of $\gF_\eps$ in $O$. Let $$f(x,v,\vv)=\eps\norm{\vv}^2 + v^2 - 2\abs{v} + 1,$$
be our integrand. That is $\gF_\eps(u) = \int_\Omega f(x,u,\nabla u)$. The EL conditions are:
 $$- \text{div}\brac{ \nabla_\vv f(x,u,\nabla u) } + \frac{d}{dv} f(x,u,\nabla u) = 0.$$
Plugging our $f$ and noting that $W'(s)=2s-2\sign(s)$, which is differentiable in $O$, we get
$$-2\eps \Delta u + 2u -2\sign(u) = 0.$$ 
In particular, \eqref{e:linear_} will be satisfied in the weak sense with any test function $\psi\in C_c^\infty(O)$.
%
Second, regularity results for elliptic operators (\eg, Corollary 8.11 in \cite{gilbarg2015elliptic}) show that $u$ is smooth in $O$ and satisfies \eqref{e:linear_} in the classical sense. \end{proof}

\begin{reptheorem}{thm:visc_eikonal}
    Let $O\subset\Omega$ be a domain as defined in Theorem \ref{thm:laplace_u}. Then, over $O$, $w_\eps$ satisfies
    \begin{equation}\label{e:w_eps_viscosity_eikonal_}
    -\sqrt{\eps}\Delta w_\eps + \sign(u_\eps) (\norm{\nabla w_\eps}^2-1) = 0
    \end{equation}
\end{reptheorem}
\begin{proof}
For brevity, we denote $u=u_\eps$, and assume $u>0$ in $O$. Then, \eqref{e:w} (in the main paper) in this case is $$w = -\sqrt{\eps}\log(1-u).$$
Let $\vx=(x_1,x_2,\ldots,x_d)$. Now,
\begin{align*}
&\frac{\partial w}{\partial x_i} = \sqrt{\eps}\frac{1}{1-u}\frac{\partial u}{\partial x_i} \\&\frac{\partial^2 w}{\partial x_i^2}=\sqrt{\eps}\frac{1}{(1-u)^2}\brac{\frac{\partial u}{\partial x_i}}^2 + \sqrt{\eps}\frac{1}{1-u}\frac{\partial^2 u}{\partial x_i^2} 
\end{align*}
Plugging in the l.h.s.~of \eqref{e:w_eps_viscosity_eikonal_} we get
\begin{align*}
\frac{-\eps}{(1-u)^2}\norm{\nabla u}^2 + \frac{-\eps}{1-u}\Delta u+\frac{\eps}{(1-u)^2}\norm{\nabla u}^2-1
\end{align*}
and this term vanishes in view of \eqref{e:linear_}.

If $u<0$ in $O$, then \eqref{e:w} (in the main paper) in this case is $$w = \sqrt{\eps}\log(1+u).$$
Similar to before:
\begin{align*}
&\frac{\partial w}{\partial x_i} = \sqrt{\eps}\frac{1}{1+u}\frac{\partial u}{\partial x_i} \\&\frac{\partial^2 w}{\partial x_i^2}=-\sqrt{\eps}\frac{1}{(1+u)^2}\brac{\frac{\partial u}{\partial x_i}}^2 + \sqrt{\eps}\frac{1}{1+u}\frac{\partial^2 u}{\partial x_i^2} 
\end{align*}
Again, plugging in the l.h.s.~of \eqref{e:w_eps_viscosity_eikonal_} we get
$$\frac{\eps}{(1+u)^2}\norm{\nabla u}^2 + \frac{-\eps}{1+u}\Delta u  + 1 - \frac{\eps}{(1+u)^2}\norm{\nabla u}^2$$
that again vanishes in view of \eqref{e:linear_}.

\end{proof}

\begin{reptheorem}{thm:var}
    Let $O$ be an open set, and $u_\eps$ a solution to \eqref{e:linear_} in $O$, $u_\eps=0$ on $\partial O$ and $u_\eps\ne 0$ in $O$. Then $ w_\eps \too  \sign(u_\eps) d_{\partial O}$ pointwise uniformly in any compact subset $O\cup\partial O$. 
\end{reptheorem}
Note that $\sign(u_\eps)$ is well defined in $O$ since we assume that $O$ does not vanish in $O$. 
\begin{proof}
First, assume $\sign(u_\eps)>0$ in $O$. Then, $u_\eps$ satisfies:
\begin{equation*}
\begin{aligned}
-\eps\Delta u_\eps + u_\eps - 1 = 0 & \qquad \text{in } O\\
u_\eps = 0 & \qquad \text{in } \partial O
\end{aligned}
\end{equation*}
Now the change of variables $v_\eps=1-u_\eps$ leads to 
\begin{equation}\label{e:var_linear}
\begin{aligned}
\frac{1}{2}\Delta v_\eps  = \frac{1}{2\eps}v_\eps  & \qquad \text{in } \Omega\\
v_\eps = 1 & \qquad \text{in } \partial\Omega
\end{aligned}
\end{equation}
Therefore, Theorem 2.3 in \cite{varadhan1967behavior} with $\lambda=\frac{1}{2\eps}$ now implies 
$$-\sqrt{\eps}\log(v_\eps)=-\sqrt{\eps}\log(1-u_\eps)\xrightarrow{\eps\too 0} d_{\partial O}$$
uniformly in compact subsets of $O\cup\partial O$. 

In the case $\sign(u_\eps)<0$ in $O$, $u_\eps$ satisfies:
\begin{equation*}
\begin{aligned}
-\eps\Delta u_\eps + u_\eps + 1 = 0 & \qquad \text{in } O\\
u_\eps = 0 & \qquad \text{in } \partial O
\end{aligned}
\end{equation*}
The change of coordinates $v_\eps = u_\eps+1$ now leads again to $v_\eps$ satisfying \eqref{e:var_linear} and invoking \cite{varadhan1967behavior} again implies that $$\sqrt{\eps}\log(v_\eps) = \sqrt{\eps}\log(1+u_\eps)\xrightarrow{\eps\too 0} -d_{\partial O}$$ 
uniformly in compact subsets of $O\cup\partial O$. Putting the two cases together and comparing to \eqref{e:w} (in the main paper) proves the theorem.
\end{proof}

\section{Networks in Sobolev spaces}
Let $f:\Real^{d}\times\Real^p\too\Real$ be a multilayer perceptron (MLP) with ReLU activation. That is, $f(\vx;\theta)$ is composed of layers of the form $\vz = \sigma(\mW\vy + \vb)$, where $\mW,\vb$ are the parameters, collectively defining $\theta\in\Real^p$, and $\sigma(s)=\max\set{0,s}$ is the ReLU applied entry-wise. We consider the functions $f(\cdot;\theta)$, for a fixed $\theta$: these are piecewise linear and continuous. In fact, each linear piece $L_k=f\vert_{\Omega_k}$ is defined over a polytope $\Omega_k\subset \Real^d$. For the analysis of the loss functions we discuss in this paper we require a complete function space that contains these neural functions and allow discussing their "derivatives" and convergence. The most natural such space is  $W^{1,p}(\Omega)$, $p\in [1,\infty)$, the Sobolev space of all $L^p(\Omega)$ functions with first weak derivatives also in $L^p(\Omega)$. We note that MLP with Softplus activation is smooth in the classical sense hence in particular belongs to $W^{1,p}(\Omega)$. 

\begin{replemma}{lem:W1}
The functions $f(\cdot;\theta)$, where $f$ is an MLP with ReLU activation belongs to $W^{1,p}(\Omega)$, for all $p\in [1,\infty]$, and $\Omega$ Lipschitz domain. 
\end{replemma}
%
\begin{proof}
Denote $f(\vx):=f(\vx;\theta)$. To show that $f\in W^{1,p}(\Omega)$ we need to show two things: $f\in L^p(\Omega)$ (which is clear since $f$ is measurable and bounded as well as $\Omega$ is bounded), and that it has \emph{weak derivatives} in $L^p(\Omega)$, that is for every $i\in [d]$ there exists a function $g\in L^p(\Omega)$ so that 
\begin{equation}\label{e:weak_derivative}
\int_\Omega f\frac{\partial \psi}{\partial x_i} = -\int_\Omega g\psi,
\end{equation}
for all compactly supported smooth functions $\psi\in C^\infty_c(\Omega)$. Let $\Omega=\cup_k \Omega_k$ be the decomposition of $\Omega$ to the subdomains where $L_k=f\vert_{\Omega_k}$ is linear. The multivariate integration by parts formula provides:
\begin{equation}\label{e:mult_int_by_parts}
\int_{\Omega_k} f\frac{\partial \psi}{\partial x_i}  = \int_{\partial\Omega_k} f\psi \ip{\ve_i,\vn} - \int_{\Omega_k}\frac{\partial f}{\partial x_i}\psi.
\end{equation}
Therefore, a natural candidate for the weak derivative $\frac{\partial f}{ \partial x_i}$ is the piecewise constant $g$ that equals $\frac{\partial L_k}{\partial x_i}$ in the interior of each $\Omega_k$. Indeed, noticing that $\ip{\ve_i,\vn}$ flips sign when the normal $\vn$ flips sign, we see that all interior contributions of the boundary integrals $\int_{\partial\Omega_k} f\psi \ip{\ve_i,\vn}$ cancel, and since $\psi$ vanishes on the boundary $\partial\Omega$, summing \eqref{e:mult_int_by_parts} over $k$ we get \eqref{e:weak_derivative}.
\end{proof}

\begin{figure*}[t]
\centering
\includegraphics[width=\textwidth]{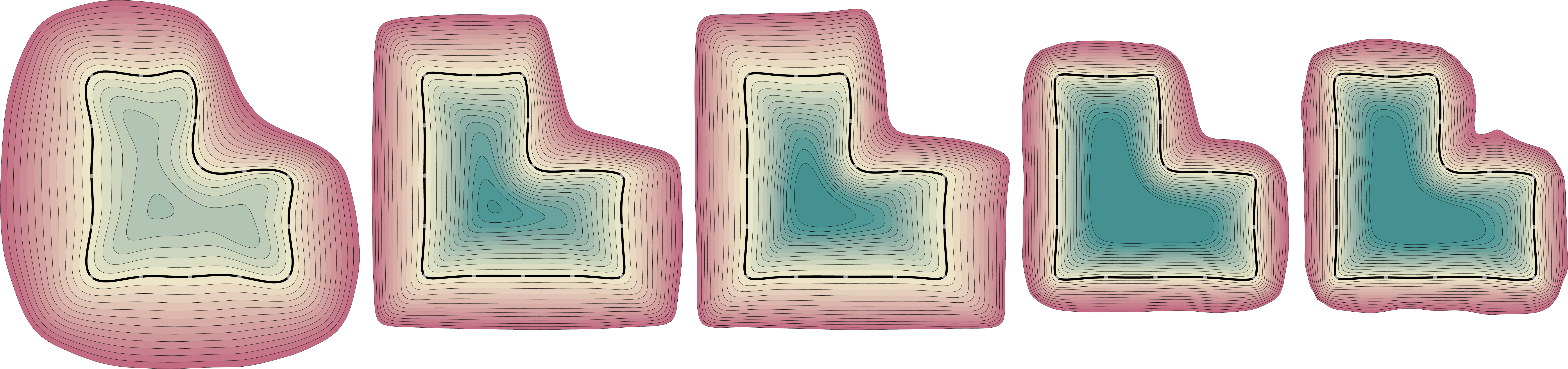}
\caption{Ablation of $\eps$. From left to right $\eps=1.0, 0.1, 0.05, 0.01, 0.005$; $\lambda=c\eps^{0.3}$ according to Theorem {\protect\ref{thm:main_gamma}}, and $c\in\set{1,10}$. }\label{fig:ablation}
\end{figure*}

\section{Experiments and implementation details}

In all the experiments we used $\eps=0.01$, and did a parameter search over $\lambda \in \set{0.2,0.3,0.5,1,10,20}$ (note that $\eps^{1/3}\approx 0.2$, and $1/3$ is in the range suggested by Theorem \ref{thm:main_gamma}), and similarly for $\mu$.  The gradients $\nabla u(\vx)$ are computed with automatic differentiation. A single training iteration with batch size $16k$ takes $0.16\text{sec}$, on an NVIDIA Quadro GP100.

\subsection{Fourier features}
 In some experiments we used Fourier features \cite{tancik2020fourfeat} as the first constant layer in the network, $\delta_k:\Real^d\too\Real^{2kd}$, where $k$ is a parameter representing the number of frequencies used. For $\vx\in\Real^d$ the Fourier feature layer is defined to be a vector $\delta_k(\vx)\in\Real^{2kd}$ with the real and imaginary parts of $\exp(\mathrm{i} 2^\omega \pi x_j)$, $\omega\in[k]$, $j\in[d]$ as entries.

\subsection{Metrics}


The distance between two point clouds $\gX_1, \gX_2$ is computed as in \cite{williams2019deep} using the standard one-sided and double-sided Chamfer and Hausdorff distances. We denote by $d^\rightarrow_C(\gX_1,\gX_2)$, $d_C(\gX_1,\gX_2)$ the one-sided and double-sided $\ell_1$ Chamfer distance; and by $d^\rightarrow_H(\gX_1,\gX_2)$, $d_H(\gX_1,\gX_2)$ the one-sided and double-sided Hausdorff distance. These are defined as follows:
\begin{align*}
    d^{\too}_C(\gX,\gY) &= \frac{1}{|\gX|}\sum_{\vx\in\gX} \min_{\vy\in\gY} \norm{\vx-\vy}_2 \\
    d_C(\gX,\gY) &= \frac{1}{2}\parr{d_C^{\too}(\gX,\gY) + d_C^{\too}(\gY,\gX)}\\
    d^{\too}_H(\gX,\gY) &= \max_{\vx\in\gX} \min_{\vy\in\gY} \norm{\vx-\vy}_2 \\
    d_H(\gX,\gY) &= \max\set{d_H^{\too}(\gX,\gY) + d_H^{\too}(\gY,\gX)},
\end{align*}
To measure the distance between a surface and a point cloud, or between two surfaces, we first sample each surface $\gS$ densely (\ie, with $10m$ uniformly random points) $\gY\subset\gS$,  and then measure the distance between the corresponding points clouds as described above.

\subsection{2D evaluation}
Figure \ref{fig:ablation} demonstrates ablation of $\eps$; we show $w_\eps$ for $\eps$ in a range of values $\eps\in\set{1.0,0.1,0.05,0.01,0.005}$, where $\lambda$ is chosen according to Theorem \ref{thm:main_gamma}, \ie, $\lambda=c\eps^{0.3}$, and we take $c\in\set{1,10}$. Note that as $\eps$ decreases $w_\eps$ is closer to a distance function. 


\begin{figure*}\scriptsize
    \centering
    \begin{tabular}{c}
         \includegraphics[width=\textwidth]{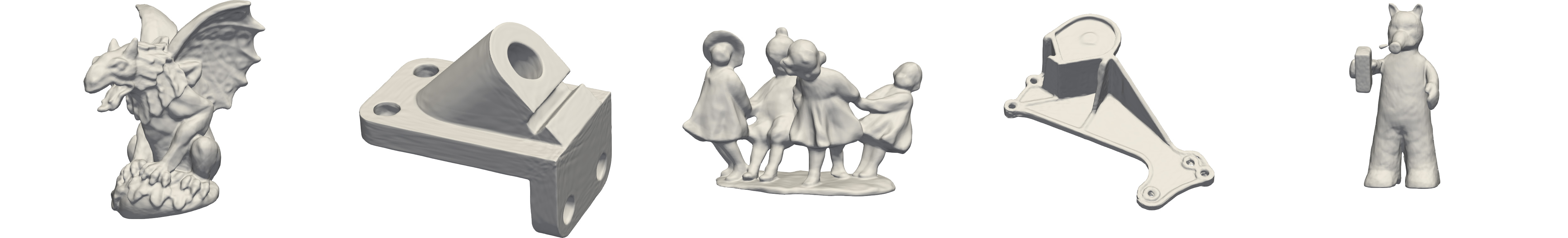} \\
         PHASE (points and normals) \\
         \includegraphics[width=\textwidth]{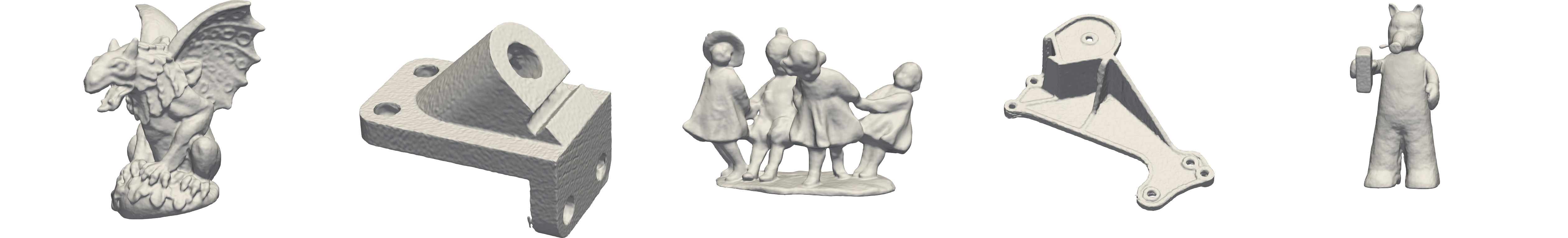} \\
         PHASE + FF (points and normals) \\
         \includegraphics[width=\textwidth]{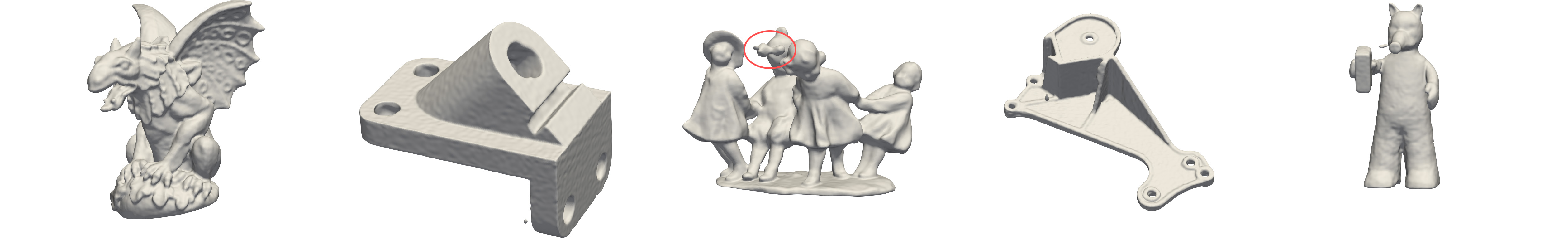}
         \\ 
         IGR + FF (points and normals) \\ \hline
         \includegraphics[width=\textwidth]{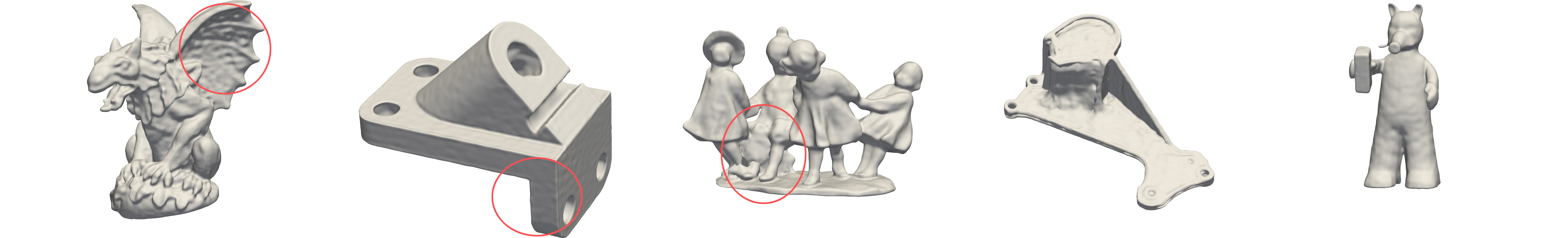}\\
         PHASE + FF (points) \\
         \includegraphics[width=\textwidth]{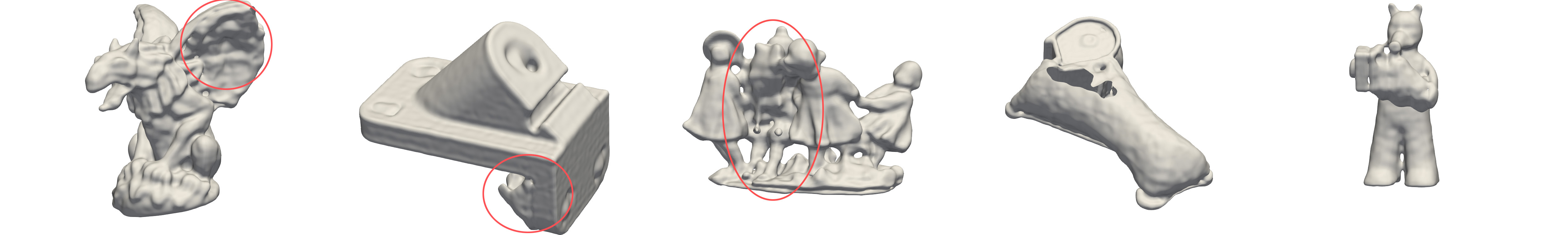}
         \\
         IGR + FF (points)  
    \end{tabular}
    \caption{Surface reconstruction benchmark \cite{williams2019deep} comparing PHASE and IGR with/without normal data and Fourier features. Note that IGR has a slight tendency to add redundant surface area that is amplified in the more challenging points-only case with Fourier features, where PHASE, due to its minimal surface perimeter property is less susceptible to high frequency artifacts. }
    \label{fig:dgp_methods}
\end{figure*}

\subsection{Surface reconstruction benchmark}
This benchmark consists of 5 models, the input (train data) to each model is a point cloud $\gX$ of size $170k-290k$, and corresponding normal data $\vn:\gX\too\gS(\Real^3)$, as well as ground-truth test point cloud $\gY$. In this experiment we used the PHASE loss in \eqref{e:loss}  (in the main paper) with the normal loss \eqref{e:N_intr} (in the main paper), and $\mu=\lambda=10$. The batch size was taken to be $16k$ and we performed $100k$ iterations, which roughly correspond to $5k-10k$ epochs. Figure \ref{fig:dgp_methods} shows qualitative results of PHASE. 

\paragraph{PHASE and Fourier Features.}
Training PHASE with Fourier features and normal data leads to comparable results to PHASE alone, but requires order of magnitude less iterations. In Table \ref{tab:phase_ff_dgp} we show the results of PHASE trained with Fourier Features ($k=6$), on the surface reconstruction benchmark of \cite{williams2019deep} (\ie, point clouds with normals), with the same parameters $\mu=\lambda=10$, as before. In this case we trained the model for $10k$ iterations with batch-size $16k$. See Table \ref{tab:dgp} in the main paper for comparison with other methods. Note that PHASE and PHASE+FF are roughly equivalent according to the distance metric scores, maybe with slight advantage to PHASE without FF. However, PHASE+FF was trained for $10k$ iterations versus $100k$ iterations in the case of PHASE without FF, and it also exhibits slightly more high frequency details not captured by the metric scores, see Figure \ref{fig:dgp_methods}. 

\subsection{Learning from point clouds and Fourier features}
In this example we again worked with the surface reconstruction benchmark but explored the more challenging case of removing the normal data and using only the point clouds $\gX$. We wanted to see the bias caused by the high frequency methods and compared PHASE and IGR in the case of adding Fourier features $k=6$ to both. For PHASE we used $\lambda=10$ as before, but smaller $\mu$, \ie, $\mu=0.5$, as we found larger $\mu$ is more unstable. For IGR we did a parameter search and chose the one that provides lowest error metrics on the Gargoyle model. Figure \ref{fig:dgp_methods} shows the results of both. Note that IGR is much more sensitive to the high frequency bias and adds extraneous surface parts. In contrast, PHASE, due to its minimal perimeter property is much less susceptible to high frequency bias, although is not completely immune to this bias for point cloud data alone, as can be inspected in the image. For point cloud and normal data, we found PHASE to be even more robust to the high frequency bias, see PHASE+FF (points and normals) experiment above.

\begin{table}[t]\scriptsize	
    \centering
    \begin{tabular}{c|c|c|c|c|c} 
               \multicolumn{2}{c}{} & \multicolumn{2}{|c|}{Ground Truth}  & \multicolumn{2}{c}{Scans}\\ \hline
         Model & Method & $d_C$ & $d_H$ & $d_C^\too$ & $d_H^\too$  \\  \hline
         \multirow{2}{*}{Anchor}  
          & PHASE+FF & 0.24 & 5.30 & 0.09 & 1.17 \\
          & PHASE & \textbf{0.21} & \textbf{4.29} & 0.09 & 1.23 \\ \hline
         \multirow{2}{*}{Daratech}  
          & PHASE+FF & \textbf{0.18} & \textbf{2.53} & 0.08 & 1.79 \\
          & PHASE & \textbf{0.18} & 2.92 & 0.08 & 1.80 \\
          \hline
         \multirow{2}{*}{DC} 
          & PHASE+FF & 0.15 & 2.32 & 0.05 & 2.77 \\
          & PHASE & 0.15 & 2.52 & 0.05 & 2.78 \\
          \hline
         \multirow{2}{*}{Gargoyle}  
          & PHASE+FF & \textbf{0.16} & 3.68 & 0.06 & 0.87 \\
          & PHASE & \textbf{0.16} & \textbf{3.14} & 0.07 & 1.09 \\
          \hline
         \multirow{2}{*}{Lord Quas}  
          & PHASE+FF & 0.12 & \textbf{0.84} & 0.04 & 0.94 \\
          & PHASE  & \textbf{0.11} & 0.96 & 0.04 & 0.96 
    \end{tabular}\vspace{-5pt}
    \caption{PHASE+FF results on the surface reconstruction benchmark of \cite{williams2019deep}. See also Table \ref{tab:dgp} in main paper for more methods.}
    \label{tab:phase_ff_dgp}
\end{table}

\subsection{Parameters scan}
Table \ref{tab:tuning}, logs the \emph{entire} range of reconstruction metric errors for \emph{all} hyperparams options $\lambda,\mu\in\set{0.2,1,10}$, trained with a smaller network: an MLP with 5 layers of 256 neurons each. Note that the range intervals are rather small (except Anchor, where the worst and best results are shown in  Figure \ref{fig:room_and_anchor}), and in fact the best results further improve the state of art in some of the cases.

\begin{table}[h]\scriptsize	
    \centering
    \resizebox{\columnwidth}{!}{%
    \begin{tabular}{c|c|c|c|c|c}                
         Metric & Anchor & Daratech & DC & Gargoyle & Lord Quas \\  \hline
         $d_C$ & [0.229, 0.457] & [0.174, 0.199] & [0.143, 0.161] & [0.158, 0.172] & [0.112, 0.127] \\ \hline
         $d_H$ & [5.02, 14.8] & [2.43, 3.68] & [1.61, 2.38] & [3.23, 4.69] & [0.77, 2.32]
     \end{tabular}\vspace{-10pt}}
    \caption{Reconstruction ranges for \emph{all combinations} of params.}
    \label{tab:tuning}
\end{table}

\begin{figure}[t]
    \centering
\includegraphics[width=0.9\columnwidth]{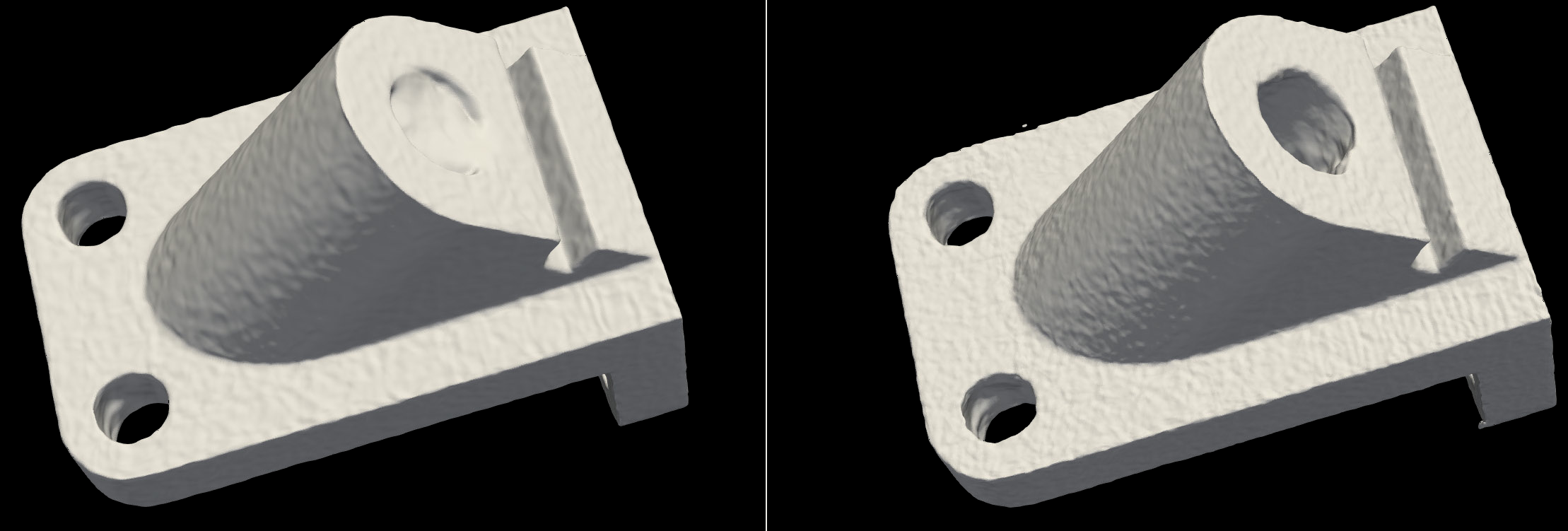}  
    \caption{The  \emph{worst} and \emph{best} examples in Table 1, Anchor.}
    \label{fig:room_and_anchor}
\end{figure}

\end{document}